\documentclass{article}
\usepackage{PRIMEarxiv}

\usepackage[numbers]{natbib}
\usepackage{amsmath}
\usepackage{amsfonts}
\usepackage{bm}
\usepackage{url}
\usepackage[section]{placeins}
\usepackage{verbatim} 

\usepackage{amsthm}
\usepackage{graphicx}
\newtheorem{theorem}{Theorem}
\newtheorem{corollary}{Corollary}
\newtheorem{definition}{Definition}

\newtheorem{remark}{Remark}
\DeclareMathOperator{\PE}{PE}

\title{Persistent Entropy as a Detector of Phase Transitions
}

\author{
  Marcos Gutiérrez-del-Pozo \\
  Universidad de Sevilla \\
  Sevilla, Spain \\
  \texttt{margutdel@alum.us.es}
  \And
   Matteo Rucco \\
  Spindox Labs \\
  Trento, Italy \\
  \texttt{ruccomatteo@gmail.com}
  \And
  Eduardo Paluzo-Hidalgo\thanks{Funded by the European Union’s Horizon 2020 research and innovation programme under the Marie Skłodowska-Curie grant agreement No. 101153039 (CHALKS), and by the IMUS–María de Maeztu grant CEX2024-001517-M (Apoyo a Unidades de Excelencia María de Maeztu), funded by MICIU/AEI/10.13039/501100011033.} \\
  Universidad de Sevilla \\
  Sevilla, Spain \\
  \texttt{epaluzo@us.es}
}

\begin{document}
\maketitle
	
	\begin{abstract}
		Persistent entropy is a scalar summary of persistence barcodes widely used to
		detect regime changes, yet there is no account of when a structural change in a
		barcode must produce a detectable change in entropy. We establish a
		model-agnostic theorem supplying such conditions. Treating persistence diagrams
		as random objects indexed by a control parameter, we identify a
		dispersion--condensation mechanism in the normalised persistence weights and
		derive an explicit lower bound on the entropy difference between the two
		regimes, valid with high probability at finite sample size and insensitive to
		the absolute scale of bar lifetimes. We also give a procedure for verifying the
		hypotheses on empirical barcodes. Applied to convolutional networks, the
		criterion shows that the circular organisation of learned filters reported by
		Gabrielsson and Carlsson emerges through a sharp topological phase transition,
		and locates its onset: within a few hundred iterations on MNIST, but an order of
		magnitude later on CIFAR-10. The same criterion detects the Kuramoto
		synchronization and Vicsek order--disorder transitions.
	
	\end{abstract}


\section{Introduction}
	
	Phase transitions are fundamental phenomena in the study of complex systems, marking qualitative reorganizations of collective behavior induced by variations of control parameters. In classical statistical physics, phase transitions are traditionally characterized through the emergence or disappearance of local order parameters within the Landau symmetry-breaking paradigm. However, many contemporary systems of interest, including coupled oscillators, active matter, biological collectives, and data-driven dynamical models, exhibit high-dimensional, noisy, and heterogeneous dynamics for which suitable order parameters may be unknown, difficult to define, or strongly model-dependent. This has motivated the search for alternative model-agnostic descriptors capable of capturing global structural changes associated with phase transitions.
	
	Topology provides a natural framework for this purpose, as topological descriptors capture global geometric and connectivity properties that remain invariant under continuous deformations. In particular, persistent homology, a central tool in topological data analysis~\cite{Edelsbrunner2002_Topological,Carlsson2009_Topology}, extracts multiscale topological information from point clouds, networks, and time-dependent data by tracking the birth and death of features across a filtration and representing them through persistence diagrams or barcodes. This approach has been successfully used to detect regime changes and phase transitions in several contexts, including synchronization dynamics~\cite{Stolz2017_FunctionalNetworks,Zabaleta2023_Rossler,Itabashi2021_TimeVariant}, collective motion~\cite{Topaz2015_Aggregation,Gonzalez2022_ScaleFree}, neural activity, and biological systems~\cite{Piangerelli2018_Seizures}. To facilitate interpretation, persistent entropy was later introduced as an information-theoretic scalar summary of persistence diagrams~\cite{chintakunta2015entropy, Rucco2016_ImmuneNetwork,Rucco2015_PersistentEntropy}. 
	
	Although persistent entropy has shown strong empirical performance in detecting regime changes~\cite{He2022_AubryAndre,Zabaleta2023_Rossler,Piangerelli2018_Seizures}, its theoretical foundations remain limited. In particular, there is no general framework explaining when a phase transition induces a structural reorganization of persistence diagrams that necessarily leads to a detectable change in persistent entropy. Since persistent entropy depends only on normalized lifetime weights, it is unclear under which structural conditions a redistribution of persistence mass produces a non-vanishing entropy gap.
	
	The relevance of this problem extends across both physics and machine learning. In physics, many contemporary phase transitions are characterized by changes in global or nonlocal organization rather than by classical local order parameters, motivating topological approaches capable of capturing structural reorganization beyond model-specific observables. At the same time, persistent entropy has emerged as a practical tool for analyzing learning dynamics, where it is used to monitor the evolution of learned representations during stochastic optimization. In both settings, persistence diagrams are naturally affected by stochasticity, finite-sample effects, and optimization noise, making a probabilistic description essential.

	The machine-learning setting supplies the motivating example for this work. Gabrielsson and Carlsson~\cite{gabrielsson2019exposition} showed that the $3\times3$ filters learned by a convolutional network, viewed as a point cloud in $\mathbb{R}^{9}$ and analyzed under a density filtration, organize around circular structures, echoing the topology previously identified in natural image patches and in primary visual cortex. Their evidence is a persistence barcode in which one $H_1$ interval visibly dominates. This is compelling but qualitative: the barcode is read by eye, at a checkpoint chosen after the fact, and there is no criterion separating a genuinely emerged cycle from a long bar produced by sampling noise. Two natural questions therefore remain open. When during training does the circular structure actually emerge? And what constitutes evidence that it has emerged, rather than merely appearing to have done so? Both questions are instances of the general problem above: the barcode passes from a configuration in which persistence is spread over many comparable intervals to one in which a single interval carries almost all of it, and what is needed is a criterion certifying that this passage has taken place.
	
	The aim of this work is to address this gap by providing a general, model-agnostic theorem establishing sufficient conditions under which persistent entropy detects phase transitions characterized by a dispersion–condensation reorganization of normalized persistence weights. Working within a probabilistic framework in which persistence diagrams are treated as random objects depending on a control parameter~\cite{Mileyko2011_Probability,Chazal2014_Convergence,Bubenik2015_Statistical}, we show that persistent entropy can separate two distinct persistence regimes: a dispersed regime, where several persistence features carry a non-negligible fraction of the persistence mass, and a condensed regime, where a small number of dominant features captures most of the mass while the diagram complexity remains bounded. Under these conditions, we derive an explicit entropy gap separating both regimes with high probability.

	\paragraph{Contributions of this work.}
	The main contributions of this work are threefold. First, we establish a model-agnostic theorem giving sufficient conditions, stated as a dispersion--condensation reorganization of normalized persistence weights, under which persistent entropy must separate two regimes by an explicit gap that holds with high probability at finite size. Because the gap depends only on normalized weights, it is insensitive to the absolute scale of bar lifetimes, and it converts the qualitative reading of a barcode into a computable guarantee. We complement it with a numerical procedure that identifies dispersed and condensed regimes directly from empirical barcodes and returns a certified lower bound on the gap.

	Second, we apply this criterion to the topology of convolutional filters, which is our principal application. Gabrielsson and Carlsson~\cite{gabrielsson2019exposition} showed that the space of learned $3\times3$ filters organizes around circular structures; we show that this organization is reached through a topological phase transition, and we determine when. Treating the training iteration as the control parameter and the barcode as a random object, we identify the interval over which the normalized persistence weights condense onto a dominant $H_1$ cycle and certify the resulting entropy separation. The onset is dataset-dependent: on MNIST the first convolutional layer condenses within a few hundred iterations, whereas on CIFAR-10 the same reorganization is delayed by more than an order of magnitude. The second convolutional layer traverses the mechanism in the opposite direction, dispersing into several coexisting cycles as training proceeds.

	Third, we confirm that the criterion is not specific to learning dynamics by detecting two classical physical transitions with it: synchronization in the Kuramoto model~\cite{Strogatz2000_Kuramoto,Ott2008_LowDimensional} and collective motion in the Vicsek model~\cite{Vicsek1995_Novel,Toner2005_Hydrodynamic}.
		
	\section{Background}\label{sec:background}
	
	In this section, we present the main concepts needed to formulate and understand the central result of this paper. We begin by recalling the definition of persistence diagrams, together with the notion of distance between them. We then consider families of persistence diagrams depending on a control parameter that governs the underlying data distribution. Finally, we introduce the definition and main properties of persistent entropy associated with a persistence diagram.

	\subsection{Persistent Homology}
	To begin with, recall that Topological Data Analysis (TDA) is a framework aimed at studying the shape of data by extracting topological features from discrete structures such as point clouds, graphs, or more general geometric representations of datasets. In order to rigorously apply these techniques, it is necessary to introduce the notion of simplicial complexes, which provide a combinatorial model for approximating the underlying topology of the data.
	
	\begin{definition}[Simplicial complex]\label{simpcomp}
		A simplicial complex $K$ on a vertex set $V$ is a collection of non-empty subsets of $V$ such that if $\tau \in K$ and $\hat{\tau} \subseteq \tau$, then $\hat{\tau} \in K$.
	\end{definition}
	
	The elements $\tau \in K$ are called simplices and they serve as the fundamental building blocks of the simplicial complex. The dimension of a simplex $\tau$ is defined as $\dim(\tau) = \#\tau - 1$, where $\#\tau$ denotes the cardinality of $\tau$. The dimension of the simplicial complex $K$ is then given by
	\[
	\dim(K) = \max\{\,\dim(\tau)\mid \tau \in K\}.
	\]
	
	When dealing with a dataset $X$ represented as a point cloud, it is necessary to impose a combinatorial structure that captures relationships between data points. This is achieved by constructing a simplicial complex whose simplices encode proximity or similarity relations among the points. In order to formalize this construction, we introduce the notion of a filter function.
	
	\begin{definition} [Filter function]
		Let $K$ be a simplicial complex. A filter function is a map $f: K \to \mathbb{R}$ that is monotone in the sense that for every $\hat{\tau} \subset \tau$, we have
		$f(\hat{\tau}) \leq f(\tau)$.
	\end{definition}
	
	A filtration of a simplicial complex is obtained by considering the sublevel sets induced by the filter function. More precisely, for each parameter $t \in \mathbb{R}$, we define a subcomplex $K_t$ by
	\[
	K_t = f^{-1}(-\infty, t] = \{\tau \in K \mid f(\tau) \in (-\infty, t]\}
	\]
	This construction yields a nested sequence of simplicial complexes $\{K_t\}_{t \in \mathbb{R}}$, where the set   $K_t$ is a simplicial complex by monotonicity of $f$. Filtrations allow us to analyze how the topology of the data evolves as the parameter $t$ varies. A fundamental example of such a filtration, particularly when the dataset $X$ is endowed with a metric structure, is the Vietoris-Rips filtration.
	
	\begin{definition} [Vietoris-Rips filtration]
		Let $(X, d_X)$ be a metric space. 
			The Vietoris--Rips filtration is defined by the filter function
			$f: K \to \mathbb{R}$ assigning to each simplex
			$\tau = \{x_0, \dots, x_m\} \in K$, with $x_i \in X$, its diameter
			\[
			f(\tau) = \max_{0\le i,j\le m} d_X(x_i, x_j),
			\]
			and by the associated sublevel sets
			$K_t = f^{-1}(-\infty,t] = \{\tau\in K \mid f(\tau)\le t\}$.
		
	\end{definition}
	
	Once such a filtration has been constructed, we can study the evolution of topological features across different scales. To this end, it is necessary to introduce algebraic tools that allow us to encode and manipulate these features. A key ingredient in this construction is the boundary operator, which in turn requires a notion of orientation on simplices.
	
	An orientation on a simplicial complex $K$ is induced by a total ordering of its vertex set. More precisely, let $V(K)$ denote the set of vertices of $K$, and let
	$o : V(K) \to \mathbb{N}$ be an injective function assigning a unique order to each vertex. A simplex $\tau = \{x_0, \ldots, x_m\} \in K$ is said to be oriented if its vertices are written so that
	\[
	o(x_0) < o(x_1) < \cdots < o(x_m).
	\]
	
	\begin{definition}[Face of a simplex]
		Let $K$ be an oriented simplicial complex and let $\tau = \{x_0, \dots, x_m\} \in K$ be an oriented simplex. The $i$-th face of $\tau$ is defined as the simplex obtained by removing the $i$-th vertex.
	\end{definition}
	
	The boundary of a simplex is then defined as the alternating sum of its faces, taking into account the orientation.
	
	\begin{definition}[Boundary operator]\label{boundary}
		Let $K$ be an oriented simplicial complex and let $\tau = \{x_0, \dots, x_m\} \in K$ be an oriented simplex. The boundary operator $\partial_{m}$ applied to $\tau$ is defined by
		\[
		\partial_{m}(\tau) = \sum_{i=0}^{m} (-1)^i \{x_0, \dots, x_{i-1}, x_{i+1}, \dots, x_m\}.
		\]
	\end{definition}
	
	It can be verified that the boundary operator satisfies the fundamental property $\partial_{m-1} \,\circ\, \partial_{m} = 0$. This property is essential for the development of homological invariants, which provide a rigorous way to track the evolution of topological features throughout the filtration.
	
	Let $\mathbb{F}$ be a field. A $m$-chain on a simplicial complex $K$ with coefficients in $\mathbb{F}$ is defined as a linear combination of $m$-simplices, that is,
	\[
	\gamma = \sum_{i} c_i \,\tau_i,
	\]
	where $\tau$ are $m$-simplices of $K$ and $c_i$ are coefficients in $\mathbb{F}$. We denote by $C_m(K,\mathbb{F})$ the group of $m$-chains on a simplicial complex $K$ with coefficients in $\mathbb{F}$. When the coefficient field is clear from the context, we simply write $C_m(K)$.
	
	The boundary operator extends linearly to chains. For a chain $\gamma \in C_m(K, \mathbb{F})$, we define $\partial_m: C_m(K) \to C_{m-1}(K)$
	by
	\[
	\partial_m(\gamma) = \sum_{i} c_i \, \partial_m(\tau_i),
	\]
	where $\partial_m(\tau_i)$ is given by Definition~\ref{boundary}.
	
	The previous constructions naturally lead to the definition of homology groups, which provide algebraic invariants that capture the topological structure of the simplicial complex. The $m$-dimensional holes of a simplicial complex $K$ are detected by considering $m$-chains whose boundary vanishes, but which are not themselves boundaries of $m+1$-chains. More concretely, the $m$-th homology group of $K$ is defined as the quotient group
	\[
	H_m(K) = \frac{\ker \partial_m}{\operatorname{im} \partial_{m+1}}
	= \frac{\{\gamma \in C_m(K) \mid \partial_m(\gamma)=0\}}
	{\{\gamma \in C_m(K) \mid \exists\, \hat{\gamma} \in C_{m+1}(K) \,\text{ such that } \,\partial_{m+1}(\hat{\gamma})=\gamma\}}.
	\]
	and its $m$-th Betti number is defined as $\beta_m = \dim_{\mathbb{F}} H_m(K)$.

	Intuitively, the Betti numbers quantify the number of independent topological features in each dimension: $\beta_0$ counts the number of connected components of $K$, $\beta_1$ counts the number of independent one-dimensional loops, $\beta_2$ counts the number of cavities, and higher-dimensional Betti numbers generalize these notions.

	When the coefficients are taken in a field, each homology group $H_m(K)$ is in fact a vector space. This algebraic structure is particularly useful because it allows us to study how homological features evolve across a filtration using linear maps.
	
	\begin{definition}[Persistent homology]
		Let $\{K_t\}_{t \in \mathbb{R}}$ be a filtration and assume that $\mathbb{F}$ is a field. Then, for each $t \in \mathbb{R}$ and $m \in \mathbb{Z}_{\ge 0}$, the homology group $H_m(K_t)$ is a vector space. For every $a < b$ and fixed $m$, the inclusion $K_a \hookrightarrow K_b$ induces a linear map
		\[
		v^{a,b}_m : H_m(K_a) \to H_m(K_b).
		\]
		The $m$-th persistent homology groups are defined as the images of these maps,
		$
		\operatorname{Im}\,(v^{a,b}_m)
		$.
		The collection $\{\operatorname{Im}\, v^{a,b}_m\}_{a<b}$ is called the $m$-th persistent homology of the filtration  $\{K_t\}_{t \in \mathbb{R}}$.
	\end{definition}
	
	We assume that the $m$-th persistent homology group is finite-dimensional for all $t \in \mathbb{R}$ and $m \in \mathbb{Z}_{\ge 0}$. Under this finiteness condition, persistent homology admits a complete discrete description in terms of intervals, leading to the notions of barcodes and persistence diagrams, which provide a concise summary of the lifetime of topological features across the filtration.

	\begin{definition}[Persistence Diagram]\label{def:pers_diag}
		A persistence diagram in a fixed homological degree is a locally finite multiset of points%
		\[
		D \subset \{(b, d)\in \mathbb{R}^2 \, : \, d-b>0\}
		\]
		together with the diagonal $
		\Delta = \{(t, t) : t \in \mathbb{R}\}$ considered with infinite multiplicity, where each pair represents the birth and death of a topological feature. The associated lifetime is $\ell = d - b$, which measures the persistence of the corresponding topological structure across the filtration.
	\end{definition}

An equivalent representation of a persistence diagram is given by its barcode, which is the
collection of intervals $[b,d)$, where each interval corresponds to a homological feature and
records its birth and death times across the filtration.
	
Persistence diagrams, either in their point-based representation or through their equivalent
barcode representation, provide a compact summary of the multiscale topological structure of the
data. To compare persistence diagrams, we work with their point-based representation and
introduce the bottleneck distance, one of the most commonly used metrics in TDA.
	
	\begin{definition}[Bottleneck Distance]
		Let $D_1$ and $D_2$ be persistence diagrams. We extend both diagrams by
		including the diagonal $\Delta$ with infinite multiplicity. The bottleneck distance between $D_1$ and
		$D_2$ is defined as
		\[
		d_B(D_1,D_2)=\inf_{\gamma}
		\sup_{x\in D_1\cup\Delta}
		\|x-\gamma(x)\|_\infty,
		\]
		where $\gamma$ ranges over all bijections  $\gamma:D_1\cup\Delta \rightarrow D_2\cup\Delta$.
		
		For two points $x=(b_1,d_1)$, $\gamma(x)=(b_2,d_2)$, the matching cost is given by
		\[
		\|x-\gamma(x)\|_\infty
		=
		\|(b_1,d_1)-(b_2,d_2)\|_\infty
		=
		\max\{|b_1-b_2|,|d_1-d_2|\}.
		\]
		
		If a point $x=(b,d)$ is matched to a point on the diagonal
		$(t,t)\in\Delta$, its cost is
		\[
		\inf_{t\in\mathbb{R}}\|(b,d)-(t,t)\|_\infty
		=
		\frac{d-b}{2}.
		\]
	\end{definition}
	
	Thus, the bottleneck distance allows persistence features to be matched
	either with features of the other diagram or with the diagonal, which
	represents features with zero persistence.

	\subsection{Persistent Entropy}\label{sec:backgroundPE}

    In many practical applications, the data are not static but instead depend on an external control parameter. This parameter may represent a physical quantity, a noise level or a sampling scale, among many other possibilities. As the parameter changes, the underlying data set evolves, and consequently its topological features may also appear, disappear, or undergo significant transformations. This perspective motivates the study of parameterized families of persistence diagrams, where the persistence diagram is regarded as a function of the external parameter.

    Let $\mathcal{D}$ denote the space of persistence diagrams and let $\lambda \in \mathbb{R}$ be an external control parameter. For each $\lambda$, we denote by $D(\lambda)\in\mathcal{D}$ the persistence diagram associated with the data corresponding to that parameter value.

Since persistent entropy is defined in terms of the lifetimes of persistence intervals, we restrict our attention to persistence diagrams with finite total persistence. Throughout this work, we therefore consider the subspace $(\mathcal{D}_F,d)$ of persistence diagrams endowed with the bottleneck distance and restricted to diagrams with finite total persistence.
	\[
	\mathcal{D}_F=
	\{
	D\in\mathcal{D}\,:\,
	L(D)=\sum_i\ell_i<\infty
	\},
	\]
	where $\ell_i=d_i-b_i>0$
	denotes the lifetime of each off-diagonal point $(b_i,d_i)$ of the persistence diagram.
	
	Persistence diagrams containing intervals of infinite length do not belong to $\mathcal{D}_F$. In such cases, one may first apply a projection $\phi:\mathcal{D}\rightarrow\mathcal{D}_F$ that maps infinite intervals to finite ones, obtaining a diagram with finite total persistence. Different choices of such projections have been considered in the literature \cite[Sec.~3.4]{Atienza2020_PersistentEntropy}.
	
	In many applications, the observed data are generated by stochastic processes, random sampling, or randomized learning algorithms. Consequently, throughout this work we model persistence diagrams as random elements. More precisely, let $(\Omega,\mathcal{F},\mathbb{P})$ be a probability space. A random persistence diagram is a measurable map
\[
D(\lambda):(\Omega,\mathcal{F},\mathbb{P})\longrightarrow \mathcal{D}_F,
\]
indexed by a control parameter $\lambda\in\mathbb{R}$.
For each realization $\omega\in\Omega$, the value
$D(\lambda,\omega)\in\mathcal{D}_F$ is a persistence diagram associated with that realization. For notational simplicity, whenever no ambiguity arises, we write $D(\lambda)$ instead of $D(\lambda,\omega)$.
	
	To summarize the distribution of lifetimes within a diagram, we introduce persistent entropy, which provides a scalar measure of how persistence is distributed among topological features.
	
	\begin{definition}[Persistent Entropy]\label{def:persentropy}
		Let $(\mathcal{D}_{F}, d)$ be the space of persistence diagrams with finite total persistence. For $D(\lambda) \in \mathcal{D}_{F}$ with lifetimes $\ell_i(\lambda)$, define normalized weights
		\[
		p_i(\lambda) = \frac{\ell_i(\lambda)}{\sum_j \ell_j(\lambda)}.
		\]
		The persistent entropy is defined as
		\[
		\mathrm{PE}\bigl(D(\lambda)\bigr) = - \sum_i p_i(\lambda) \log p_i(\lambda).
		\]
	\end{definition}
	
	Persistent entropy ($\PE$) quantifies the dispersion of persistence among features. If a single feature dominates (one $\ell_i$ is much larger than the others), then $p_i\approx 1$ and the entropy is close to zero. In contrast, when the lifetimes are more evenly distributed, the entropy increases and is maximized when all weights are equal. In particular, for a diagram $D(\lambda)$ with $N$ off-diagonal points, one has
	\[
	0 \leq \mathrm{PE}\bigl(D(\lambda)\bigr) \leq \log N,
	\]
	with equality on the right-hand side if and only if $p_i = 1/N$ for all $i$. This motivates the normalized persistent entropy
	\[
	\mathrm{NPE}\bigl(D(\lambda)\bigr) = \frac{\mathrm{PE}\bigl(D(\lambda)\bigr)}{\log N},
	\]
	which takes values in $[0,1]$. For $N=1$ we set $\mathrm{NPE}\bigl(D(\lambda)\bigr)=0$.
	
	\begin{remark}
		If $D(\lambda)$ has no off-diagonal points, we set $\mathrm{PE}\bigl(D(\lambda)\bigr)= 0$.
	\end{remark}
	
	\begin{remark}[Stability of Persistent Entropy]\label{rmk:stab}
		From a stability perspective, persistent entropy is not continuous with respect to the bottleneck distance in $\mathcal{D}$, even when restricted to the space $\mathcal{D}_{F}$. This is due to the fact that arbitrarily many intervals with very small persistence, close to the diagonal, can be created or removed at negligible cost in bottleneck distance while still affecting the entropy. However, in $\mathcal{D}_{F}$, when the persistence diagram contains a finite number of intervals, persistent entropy is continuous with respect to the bottleneck distance \cite[Prop.~9]{Atienza2017}. In practice, a way to recover this stability is through truncation of low persistence intervals: given a threshold $\beta > 0$ and the diagram $D(\lambda)\in\mathcal{D}_{F}$, one removes all points with persistence smaller than $\beta$, obtaining a truncated diagram $D(\lambda)^{> \beta}$ with persistence greater than $\beta$. Therefore,
		\[
		\#\,D(\lambda)^{>\beta}\;<\;\frac{L\bigl(D(\lambda)^{>\beta}\bigr)}{\beta}\;<\;\infty,
		\]
		so the truncated diagram contains a finite number of intervals.
	\end{remark}

	\section{Persistent entropy as detector of phase transitions}
	Phase transitions are detected by identifying qualitative
	changes in the persistence diagram as the control parameter crosses a
	critical value. Intuitively, such transitions correspond to structural reorganizations in the topological features encoded by the diagram. In particular, when these changes involve the appearance or disappearance of macroscopic persistent features, global summaries such as persistent entropy are expected to capture them.
	
	The theorem below formalizes this idea by linking structural changes in persistence diagrams with variations in persistent entropy. Throughout this section, we denote by $(\mathcal{D}_{F},d)$ the space of persistence diagrams with finite total persistence equipped with the bottleneck distance, following the notation of Section~\ref{sec:backgroundPE}. Furthermore, for the following definition and theorem, we will consider persistence diagrams with a finite number of off-diagonal points. As discussed in Remark~\ref{rmk:stab}, this condition can be ensured by truncating intervals of low persistence.

	Phase transitions of interest here correspond to qualitative changes in the
	persistence diagram structure at the boundary between two regimes. Not every such
	transition is detectable by persistent entropy. The following definition
	isolates the structural conditions under which persistent entropy provably separates the two phases.

		\begin{definition}[Dispersion--condensation structure]
			\label{def:disp_cond}
			Let $D(\lambda)\in\mathcal{D}_{F}$ be a random persistence diagram parametrized
			by $\lambda$, and assume throughout that its normalized persistence weights are
			indexed in non-increasing order,
			\[
			p_{1}(\lambda) \ge p_{2}(\lambda) \ge \dots \ge p_{S(\lambda)}(\lambda),
			\]
			where $S(\lambda)$ is the number of off-diagonal points of the diagram.
			Let $I^-$ and $I^+$ be two disjoint parameter regions associated with the
			dispersed and condensed regimes, respectively. Given $c\in(0,1)$ and
			$m,k,S\in\mathbb{Z}^{+}$ with $1\le k<m$ and $mc\le 1$, we say that $D(\lambda)$
			undergoes a \emph{dispersion--condensation transition} with parameters
			$(c,m,k,S)$ if:
			\begin{itemize}
				\item \textbf{(Dispersion)}
				For all $\lambda \in I^{-}$, the diagram $D(\lambda)$ has at least $m$
				off-diagonal points whose normalized persistence weights satisfy
				\[
				p_{m}(\lambda) \;\ge\; c .
				\]
				
				\item \textbf{(Condensation)}
				For all $\lambda \in I^{+}$, the diagram size is bounded by $S(\lambda) \le S$
				and the cumulative weight of the $k$ dominant features satisfies
				\[
				M_k(\lambda) \;:=\; \sum_{i=1}^k p_i(\lambda) \;>\; 1 - c .
				\]
			\end{itemize}
		\end{definition}

        \begin{remark}\label{auto_disp}
        The condition $mc\le 1$ is forced by the dispersion hypothesis, since $m$ weights
			bounded below by $c$ cannot sum to more than one; it guarantees in particular
			that $1-(m-1)c \ge c > 0$.
        \end{remark}
        
	\begin{remark}\label{rmk:auto_cond}
			Because the weights are sorted in non-increasing order and sum to one, the top
			$k$ of them always carry at least the average share, i.e.
			$M_k(\lambda) \ge k/S(\lambda)$ whenever $S(\lambda)>k$. This inequality is used
			in the proof of Theorem~\ref{thm:finite} and requires no additional assumption.
		\end{remark}
	
	This definition captures two qualitatively distinct regimes: one in which persistence is distributed across multiple relevant features, a dispersed phase; and another in which it is concentrated in only a few dominant ones, a condensed phase. Within this framework, the problem of phase detection reduces to identifying observables capable of distinguishing between these two behaviors.

	\begin{theorem}[Finite-size persistent entropy gap for dispersion–condensation transitions]
		\label{thm:finite}
		Suppose $D(\lambda)$ undergoes a dispersion--condensation transition with parameters $(c,m,k,S)$ over the disjoint regimes $I^-$ and $I^+$, in the sense of Definition~\ref{def:disp_cond}. Furthermore, assume
		that for $\delta\in(0,1)$ the dispersion and condensation conditions hold with
		probability at least $1-\delta$.
		Let $g(x)=-x\log x$ for $x\in (0,1]$ and define

		\begin{align}
			H_{\mathrm{disp}}(c,m)&\;=\;-\bigl(1-(m{-}1)c\bigr)\,\log\bigl(1-(m{-}1)c\bigr)
			\;+\;(m{-}1)\,c\,\log\!\tfrac{1}{c},\label{eq:Hdisp} \\[8pt]
			H_{\mathrm{cond}}(c, k, S) &= 
			\begin{cases} 
				\log(S) & \text{if } S < k \\
				\log(k) & \text{if } k\le S \le k\left(1+ e^{\frac{-g(c)-g(1-c)}{c}}\right) \\
				-(1-c) \log\left(\frac{1-c}{k}\right) - c \log\left(\frac{c}{S-k}\right) & \text{if } S > k\bigl(1+ e^{\frac{-g(c)-g(1-c)}{c}}\bigr)\textbf{}
			\end{cases} \label{eq:Hcond}
		\end{align}
		
		Assume $H_{\mathrm{disp}}(c,m)> H_{\mathrm{cond}}(c,k,S)$. We define the entropy gap
		\[
		\,\Delta(c,m,k,S) := H_{\mathrm{disp}}(c,m) - H_{\mathrm{cond}}(c,k,S) > 0
		\]
		Then, for any $\lambda_{I^-}\in I^-$ and  $\lambda_{I^+}\in I^+$, the persistence entropy satisfies
		\[
		\mathbb{P}\!\Bigl(
		\PE\bigl(D(\lambda_{I^-})\bigr)
		- \PE\bigl(D(\lambda_{I^+})\bigr)
		\;\ge\; \Delta(c,m,k,S)
		\Bigr)
		\;\ge\; 1 - 2\delta.
		\]
	\end{theorem}
	
	\begin{proof}
		
		\textbf{Step 1: Lower bound in the disordered phase.}
		
		Let $D(\lambda)\in \mathcal{D}_{F}$ and $\lambda\in I^{-}$ satisfying dispersion conditions with parameters $(c,m,k,S)$ in the sense of Definition \ref{def:disp_cond}. Let $m(\lambda)\ge m$ be the number of weights satisfying $p_i:=p_i(\lambda)\ge c$;
			by the sorting convention of Definition~\ref{def:disp_cond} these are
			$p_1\ge p_2\ge\dots\ge p_{m(\lambda)}$, and in particular $p_1=\max_i p_i$.%

Let $g(x)=-x\log x$ for $x\in(0,1]$. Since $g$ is strictly concave in $(0,1]$, $g(a+b)<g(a)+g(b)$ for $a,b>0$, which means that merging positive mass from one coordinate into another
		strictly decreases the entropy.
		
		Writing the persistence entropy as
		\[
		\mathrm{PE}\bigl(D(\lambda)\bigr)=\sum_{p_i\geq c}g(p_i)+ \sum_{p_i<c}g(p_i)=\sum_{i=1}^{m(\lambda)} g(p_i)+\sum_{p_i<c}g(p_i),
		\]
		we may repeatedly merge all unconstrained coordinates into the coordinate $p_1$. Hence,
		\[
		\mathrm{PE}\bigl(D(\lambda)\bigr) \ge g\bigl(p_1+ \sum_{p_i<c} p_i\bigr)+\sum_{i=2}^{m(\lambda)} g(p_i)
		\]
        where equality holds if and only if there are no weights below $c$.
        
		We now repeat the same mass-merging argument for the coordinates $p_2,\dots, p_{m(\lambda)}$. More precisely, we transfer the excess mass $p_{i}-c \ge 0$ for $i\in\{2,\dots,m(\lambda)\}$ into $p_1$. Therefore,
		\[
		\mathrm{PE}\bigl(D(\lambda)\bigr)\ge g\bigl(p_1+ \sum_{p_i<c} p_i+ \sum_{i=2}^{m(\lambda)} (p_i-c) \bigr)+\sum_{i=2}^{m(\lambda)} g(c)
		\]
        with equality if and only if $p_{i}=c$ for $i\in\{2,\dots,m(\lambda)\}$.
        
		Using $1=\sum_{i=1}^{m(\lambda)}p_{i}+\sum_{p_i<c}p_i$, we obtain
		\[
		\begin{aligned}
			\PE\bigl(D(\lambda)\bigr) &\ge g\bigl(1 - (m(\lambda) - 1)\,c\bigr) + (m(\lambda) - 1)g(c) \\
			&= -\bigl(1 - (m(\lambda) - 1)\,c\bigr) \log\bigl(1 - (m(\lambda) - 1)\,c\bigr) + (m(\lambda) - 1)\,c \log\bigl(\frac{1}{c}\bigr) \\
			&:= H_{\mathrm{disp}}(c, m(\lambda))
		\end{aligned}
		\]
        
		 Furthermore, we observe that $H_{\mathrm{disp}}\left(c, m(\lambda)\right)$ is strictly increasing with respect to $m(\lambda)$. Taking the derivative with respect to $m(\lambda)$:
		\[
		\frac{\partial}{\partial m(\lambda)} H_{\mathrm{disp}}\left(c, m(\lambda)\right) = c \log\left( \frac{1-(m(\lambda)-1)\,c}{c} \right) + c.
		\]
		Applying the inequality $m(\lambda)c\le1$, which follows from the dispersion hypothesis, and Remark \ref{auto_disp}, we obtain
        \[
        \frac{1-(m(\lambda)-1)c}{c}\ge 1. 
        \]
        Consequently, $\frac{\partial H_{\mathrm{disp}}}{\partial m(\lambda)} \ge c > 0$, so $H_{\mathrm{disp}}$ is strictly increasing in $m(\lambda)$. Since $m(\lambda) \ge m$,
        \[
        H_{\mathrm{disp}}(c,m(\lambda))\ge H_{\mathrm{disp}}(c,m)=-\bigl(1-(m{-}1)c\bigr)\,\log\bigl(1-(m{-}1)c\bigr)
		\;+\;(m{-}1)\,c\,\log\!\tfrac{1}{c}
        \]

        Therefore,
			\[
			\PE\bigl(D(\lambda)\bigr) \ge H_{\mathrm{disp}}(c, m),
			\]
		
		\medskip\noindent
		\textbf{Step 2: Upper bound in the ordered phase.}
		
		On the condensation event, let $D(\lambda)\in\mathcal{D}_{F}$ with $\lambda \in I^+$. By Definition \ref{def:disp_cond}, the diagram size is bounded above by $S(\lambda) \le S$ and the cumulative mass of the $k$ dominant features $M_k(\lambda) = \sum_{i=1}^k p_i(\lambda)$ fulfills
		\[
		M_k(\lambda)>1-c,
		\]
		
		Let $g(x)=-x\log x$ for $x\in(0,1]$ and  $\PE(D(\lambda)) = \sum_{i=1}^{S(\lambda)} g(p_i)$ the persistence entropy for $p_i:=p_i(\lambda)$. Since the function $g$ is concave, entropy is maximized when
		this remaining mass is distributed uniformly among weights.
		
		\medskip
		
		\noindent\textbf{Case 1. $S(\lambda)\leq k$.}\\
		In this case the diagram contains at most $k$ persistence intervals. The entropy is maximized by the uniform distribution,
		$p_i=\frac1{S(\lambda)}$, and therefore
		\[
		\mathrm{PE}\bigl(D(\lambda)\bigr)\le \log \bigl(S(\lambda)\bigr) \leq\log(k)
		\]
		
		\medskip
		
		\noindent\textbf{Case 2. $S(\lambda)> k$.}\\
		Now there are $k$ dominant features and $S(\lambda)-k$ background features. Since $M_k(\lambda)>1-c$, the remaining features carry total mass
		\[
		\sum_{i=k+1}^{S(\lambda)}p_i(\lambda) = 1-M_k(\lambda)<c.
		\]
		Thus, allocating the mass uniformly within the top $k$ features ($p_i = M_k(\lambda)/k$ for $i \le k$) and within the remaining background features ($p_i = (1 - M_k(\lambda))/(S(\lambda) - k)$ for $i > k$) yields the conditional upper bound:
		\begin{align*}
			\mathrm{PE}\bigl(D(\lambda)\bigr) &\le \sum_{i=1}^k g\left(\frac{M_k(\lambda)}{k}\right) + \sum_{j=k+1}^{S(\lambda)} g\left(\frac{1 - M_k(\lambda)}{S(\lambda) - k}\right) \\
			&= k \left( -\frac{M_k(\lambda)}{k} \log \frac{M_k(\lambda)}{k} \right) + (S(\lambda) - k) \left( -\frac{1 - M_k(\lambda)}{S(\lambda) - k} \log \frac{1 - M_k(\lambda)}{S(\lambda) - k} \right) \\
			&= -M_k(\lambda) \log\left(\frac{M_k(\lambda)}{k}\right) - (1 - M_k(\lambda)) \log\left(\frac{1 - M_k(\lambda)}{S(\lambda) - k}\right) \\
			&:= h\left(k, M_k(\lambda), S(\lambda)\right).
		\end{align*}
		
		We now analyze the monotonicity of the functional $h$ with respect to $M_k(\lambda)$:
		\begin{align*}
			\frac{\partial}{\partial M_k(\lambda) } h\left(k, M_k(\lambda), S(\lambda) \right) = \log\left( \frac{k(1 - M_k(\lambda))}{M_k(\lambda)(S(\lambda) - k)} \right)
		\end{align*}
		The derivative is non-positive precisely when $M_k\ge\frac{k}{S(\lambda)}$, which holds automatically for sorted weights when $S(\lambda)>k$ (Remark~\ref{rmk:auto_cond}). Therefore, $h$ is decreasing with respect to $M_k(\lambda)$, so its maximum over the admissible region is attained at the smallest allowed value, $M_k\to 1-c$.
		
		Keeping $M_k=1-c$ fixed, we differentiate with respect to the diagram size $S(\lambda)$:
		\begin{align*}
			\frac{\partial}{\partial S(\lambda) } h\bigl(k, 1-c, S(\lambda) \bigr) =  \frac{c}{S(\lambda)-k}>0
		\end{align*}
		Hence the entropy is increasing with the diagram size, and using the global bound $S(\lambda)\le S$ we obtain
		\[
		\mathrm{PE}\bigl(D(\lambda)\bigr)
		\le
		-(1-c)\log\!\left(\frac{1-c}{k}\right)
		-
		c\log\!\left(\frac{c}{S-k}\right).
		\]
		
		We have therefore obtained two possible upper bounds for entropy:
		\begin{itemize}
			\item $\log(k)$ when $S(\lambda)\le k$
			\item $-(1-c)\log\!\left(\frac{1-c}{k}\right)
			-
			c\log\!\left(\frac{c}{S-k}\right)$
			when $S(\lambda)>k$.
		\end{itemize}

		If in addition $S<k$, then $S(\lambda)\le S<k$, so only Case~1 can occur and
		the first bound sharpens to $\log(S)$.

		Comparing both expressions, one finds that
		\[
		\log k \le -(1-c)\log\!\left(\frac{1-c}{k}\right) - c\log\!\left(\frac{c}{S-k}\right)
		\]
		if and only if $S \ge k\left(1+ e^{\frac{-g(c)-g(1-c)}{c}}\right)$.
		
		Consequently, we define the upper bound $H_{\mathrm{cond}}(c, k, S)$ as
		\begin{align*}
			H_{\mathrm{cond}}(c, k, S) = 
			\begin{cases}
				\log(S) & \text{if } S < k \\
				\log(k) & \text{if } k\le S \le k\left(1+ e^{\frac{-g(c)-g(1-c)}{c}}\right) \\
				-(1-c) \log\left(\frac{1-c}{k}\right) - c \log\left(\frac{c}{S-k}\right) & \text{if } S > k\bigl(1+ e^{\frac{-g(c)-g(1-c)}{c}}\bigr)\textbf{}
			\end{cases}
		\end{align*}
		and on the condensation event, we conclude,
		\[
		\PE\bigl(D(\lambda)\bigr) \le H_{\mathrm{cond}}(c, k, S).
		\]
		\medskip\noindent
		\textbf{Step 3: Entropy gap.}
		
		Let $\lambda_{I^-}\in I^-$ and $\lambda_{I^+}\in I^+$. By the union bound, the dispersion and condensation events occur simultaneously with probability at least $1-2\delta$.
		On their intersection,
		\[
		\PE\bigl(D(\lambda_{I^-})\bigr) - \PE\bigl(D(\lambda_{I^+})\bigr)
		\;\ge\;
		H_{\mathrm{disp}}(c,m) - H_{\mathrm{cond}}(c,k,S)
		\;=\;
		\Delta(c,m,k,S).
		\]
		Positivity of the gap holds by the standing hypothesis $H_{\mathrm{disp}}(c,m)>H_{\mathrm{cond}}(c,k,S)$.
		\qed
	\end{proof}
	
	\begin{remark}[Role of the constant $c$]
			The same constant governs both hypotheses: it is the floor on individual weights
			in the dispersed regime and the tolerated deficit in the condensed one. Since
			dispersion forces $c\le 1/m$, condensation then demands
			$M_k(\lambda)>1-1/m$, which is a stringent requirement when $m$ is large. The
			argument goes through verbatim with two independent constants
			$c_{\mathrm{disp}}$ and $c_{\mathrm{cond}}$, provided that $c_{\mathrm{disp}}>c_{\mathrm{cond}}$, yielding the gap
			$H_{\mathrm{disp}}(c_{\mathrm{disp}},m)-H_{\mathrm{cond}}(c_{\mathrm{cond}},k,S)$;
			we keep a single constant for readability.
	\end{remark}

		\begin{remark}[\textbf{Mechanism of detection}]
			\label{rmk:mechanism}
			The theorem identifies a precise barcode-level mechanism underlying the empirical
			success of persistent entropy in detecting regime changes: $\PE$ detects phase
			transitions whenever the transition transforms the shape of the lifetime distribution from a dispersed configuration, in which persistence
			mass is shared among multiple bars of comparable weight, to a condensed configuration, in which some bars dominate a topologically simple diagram.
			Because $\PE$ is the Shannon entropy of this distribution, it is maximally sensitive
			to this reorganization: dispersion produces high entropy, condensation produces
			low entropy, and the gap between them is controlled by explicit, computable
			constants that depend only on the concentration parameter~$c$, the number of
			dispersed bars~$m$, the number of condensed bars $k$ and the diagram complexity~$S$.
			Crucially, this mechanism operates entirely at the level of normalized weights
			and is therefore insensitive to the absolute scale of bar lifetimes.
		\end{remark}
		
		\begin{remark}[\textbf{Model-agnosticism}]
			\label{rmk:model_independence}
			The theorem is model-agnostic.
			The hypotheses, dispersion and condensation of normalized persistence
			weights, are stated purely in terms of the weight distribution of the barcode
			and make no reference to any specific dynamical model, spatial dimension,
			filtration type, or homological degree.
			Detection depends on the choice of observable and filtration only insofar as
			these choices determine whether the resulting barcodes satisfy the dispersion and
			condensation conditions.
			Whenever a phase transition reorganizes the barcode from a multi-bar, broadly
			distributed weight profile to a single-bar-dominated, low-complexity profile, $\mathrm{PE}$
			must detect it regardless of the underlying dataset.
		\end{remark}
			
		The finite-size result of Theorem~\ref{thm:finite} extends naturally to the
		asymptotic regime provided the persistence diagrams converge and a mild
		regularity condition excludes accumulation of infinitesimally short bars.
		
		\begin{corollary}[Asymptotic persistent entropy gap for dispersion–condensation transitions]
			\label{cor:asymptotic}

				Let $D_N(\lambda) \in \mathcal{D}_{F}$ denote a random persistence diagram of a
				dataset of size $N$ with control parameter $\lambda$, and suppose there exists a
				deterministic diagram $D_\infty(\lambda)\in\mathcal{D}_{F}$ such that, for each
				$\lambda$,
				\[
				d\bigl(D_N(\lambda),\, D_\infty(\lambda)\bigr)
				\;\xrightarrow[N\to\infty]{\mathbb{P}}\; 0 .
				\]
				Fix a threshold $\beta>0$ that is \emph{generic} for the limit, in the sense that
				no interval of $D_{\infty}(\lambda)$ has persistence exactly $\beta$, and
				consider the truncated diagrams $D_N(\lambda)^{>\beta}$ of
				Remark~\ref{rmk:stab}. Let $\{\delta_{N}\}_{N\in\mathbb{N}}$ be a sequence with
				$\delta_N \to 0$ as $N\to\infty$, and suppose that the hypotheses of
				Theorem~\ref{thm:finite} hold for $D_N(\lambda)^{>\beta}$ with parameters
				$(c,m,k,S)$ and probability at least $1-\delta_N$.
				Then, for any $\lambda_{I^-}\in I^{-}$ and $\lambda_{I^{+}}\in I^{+}$,
				writing
				\[
				X_N:=\PE\bigl(D_N(\lambda_{I^-})^{>\beta}\bigr)-\PE\bigl(D_N(\lambda_{I^+})^{>\beta}\bigr),
				\]
				the finite-size bound $\mathbb{P}\bigl(X_N \ge \Delta(c,m,k,S)\bigr)\ge 1-2\delta_N$
				holds for every $N$, and in the limit
				\[
				\mathbb{P}\Bigl(
				\PE\bigl(D_\infty(\lambda_{I^-})^{>\beta}\bigr)
				- \PE\bigl(D_\infty(\lambda_{I^+})^{>\beta}\bigr)
				\;\ge\; \Delta(c,m,k,S)
				\Bigr) = 1 .
				\]
		\end{corollary}
		
		\begin{proof}
			By the bottleneck stability theorem, convergence $d_\infty(D_N(\lambda), D_\infty(\lambda)) \to 0$ in probability implies that, for each $\beta>0$,
			the bars of the truncated diagrams $D_N(\lambda)^{>\beta}$ converge in length to those of $D_\infty(\lambda)^{>\beta}$, because no bar of $D_\infty(\lambda)$ has persistence exactly $\beta$. 
			
			Persistent entropy is continuous with respect to the bottleneck distance on diagrams with finite number of bars (Remark~\ref{rmk:stab}).  
			Hence, for each $\lambda$,
			\[
			\PE\bigl(D_N(\lambda)^{>\beta}\bigr) \xrightarrow[N\to\infty]{\mathbb{P}} \PE\bigl(D_\infty(\lambda)^{>\beta}\bigr).
			\]
			
			Define for any $\lambda_{I^-}\in I^{-}$ and $\lambda_{I^{+}}\in I^{+}$,
			\[
			X_{N}:=\PE\bigl(D_N(\lambda_{I^-})^{>\beta}\bigr)-\PE\bigl(D_N(\lambda_{I^+})^{>\beta}\bigr)\,\,\,\,\,\text{and}\,\,\,\,\,X_\infty:=\PE\bigl(D_\infty(\lambda_{I^-})^{>\beta}\bigr)-\PE\bigl(D_\infty(\lambda_{I^+})^{>\beta}\bigr).
			\]

				Since $X_N \xrightarrow{\mathbb{P}} X_\infty$, convergence in probability implies
				convergence in distribution, and the set $[\Delta(c,m,k,S),\infty)$ is closed.
				The Portmanteau theorem therefore gives
				\[
				\mathbb{P}\bigl(X_\infty \ge \Delta(c,m,k,S)\bigr)
				\;\ge\;
				\limsup_{N\to\infty}\,\mathbb{P}\bigl(X_N \ge \Delta(c,m,k,S)\bigr)
				\;\ge\;
				\lim_{N\to\infty}(1-2\delta_N)
				\;=\; 1,
				\]
				which concludes the proof.
			\qed
		\end{proof}

		\section{Numerical experiments}\label{sec:experiments}
		
		In this section, we provide an empirical validation of our theoretical framework by investigating three distinct classes of dynamical systems, each exhibiting qualitatively different phase transition profiles: neural network training dynamics, the Kuramoto model for synchronization, and the Vicsek model for collective motion.
		For each individual case study, we begin by detailing the underlying model configurations and experimental setups. We then analyze the evolution of the resulting random persistence diagrams as the control parameter varies, and empirically
		assess the dispersion and condensation hypotheses introduced in
		Definition \ref{def:disp_cond}. These experiments illustrate the applicability of
		Theorem \ref{thm:finite} by showing that, in representative systems where its
		hypotheses are satisfied, persistent entropy exhibits the predicted
		separation between phases. Before analyzing each case study, we introduce the dynamical methodology used to operationalize our theoretical results in practice.
		
		\subsection{Numerical Detection of Dispersion–Condensation Phases}\label{sec:num_detection}
		
		Building on the theoretical framework established in Definition \ref{def:disp_cond}, we propose a numerical procedure to identify the dispersion and condensation regimes in practice. The method is based on the evolution of the normalized persistence weights, $p_i(\lambda)$, as a function of the control parameter $\lambda \in \mathbb{R}$. Specifically, the procedure searches for the parameters that maximize the entropy separation predicted by the theory while simultaneously maximizing the proportion of persistence diagrams computed along the control parameter $\lambda$, that can be classified as either dispersed or condensed.
		
		\subsubsection{Phase Classification Criteria}
		
		Given a random persistence diagram $D(\lambda)\in\mathcal{D}_{F}$, fix a maximum value $k_{\max}$, a threshold $k\in\{1,\ldots,k_{\max}\}$, and a cut-off parameter $c\in(0,1/2]$. Since Definition \ref{def:disp_cond} requires $m>k\geq 1$ and $mc\leq 1$, only values $c\leq 1/2$ need to be considered. 
        
        We define the following empirical observables from the sorted normalized persistence weights:
		\begin{itemize}
			\item $m_{c}(\lambda) = \#\{i : p_i(\lambda) \geq c\}$, which counts the number of persistent features whose normalized weight exceeds the significance threshold.
			
			\item $M_{k}(\lambda) = \sum_{i=1}^{k} p_i(\lambda)$, representing the cumulative weight carried by the $k$ dominant features.
			
			\item $S(\lambda) = \#\{i : p_i(\lambda) > 0\}$, representing the total number of off-diagonal points in the diagram.
		\end{itemize}
		
		Let
		\[
		m_{c}^{\max}=\max_{\lambda}\{m_{c}(\lambda)\}.
		\]
		For each $m\in \{k+1,\dots,m_{c}^{\max}\}$, the diagram $D(\lambda)$ is assigned a phase label $\sigma_{c,m,k}(\lambda)\in\{-1,0,1\}$ according to the following decision rules:
		\begin{equation*}
			\sigma_{c,m,k}(\lambda)=
			\begin{cases}
				-1 & \text{if } m_{c}(\lambda)\geq m
				\quad (\text{Dispersion Phase, } I^{-}),\\[2mm]
				+1 & \text{if } M_{k}(\lambda)>1-c
				\quad (\text{Condensation Phase, } I^{+}),\\[2mm]
				
				0 & \text{otherwise}
				\quad (\text{Transition/Undefined, } I^{0}).
			\end{cases}
		\end{equation*}
		
		We denote by $I^{-}(c,m,k)$, $I^{+}(c,m,k)$ and $I^{0}(c,m,k)$ the corresponding regimes.
		
		From every admissible tuple $(c,m,k)$, we define
		\[
		S_{c,m,k}=\max_{\lambda\in I^{+}(c,m,k)} S(\lambda).
		\]
		
		According to Theorem~\ref{thm:finite}, the detected entropy gap must remain strictly positive:
		\[
		\Delta(c,m,k,S_{c,m,k})
		=
		H_{\text{disp}}(c,m)
		-
		H_{\text{cond}}(c,k,S_{c,m,k})
		>0,
		\]
		where $H_{\text{disp}}$ and $H_{\text{cond}}$ denote the theoretical entropy bounds defined in the theorem.
		
		\subsubsection{Optimal Threshold Selection}
		
		The optimal parameters $(c^*,m^*,k^*)$ are selected by balancing two complementary objectives. First, the entropy gap should be as large as possible in order to maximize the theoretical separation between the two phases. Second, the resulting classification should assign the largest possible fraction of persistence diagrams to either the dispersion or the condensation regime.
		
		For every admissible parameter triple we define the classification coverage:
		\[
		\operatorname{Cov}_{c,m,k}=\frac{\# \{\lambda\,:\,\sigma_{c,m,k}\neq 0\}}{\#\{\lambda\}}
		\]
		
		Among all admissible parameter tuples, that is, those satisfying the feasibility
		constraints above together with the optional density constraints of
		Remark~\ref{rmk:opt_const}, and such that
		$\Delta(c,m,k,S_{c,m,k})>0$, we maximize the normalized objective function
		\[
		\Gamma_{c,m,k}
		=
		\frac{\Delta(c,m,k,S_{c,m,k})}
		{\log(S_{\max})}
		\,
		\operatorname{Cov}_{c,m,k},
		\]
		where
		\[
		S_{\max}
		=
		\max_{\lambda}S(\lambda)
		\]
		is the maximum number of persistence points observed in the dataset.
        The normalization by $\log(S_{\max})$ rescales the entropy gap to the same range
as the coverage term, allowing both contributions to have a comparable influence
in the objective function.

		The optimal threshold parameters $(c^*,m^*,k^*)$ are then obtained by solving the maximization problem
		\[
		(c^*,m^*,k^*)
		=
				\arg\max_{\substack{
                            0<c\le 1/2\\
							1\leq k\leq k_{\max}\\
							k+1\leq m\leq m^{\max}_{c},\\ 
							mc\le 1
				}}
				\Gamma_{c,m,k}.
				\]
				
				This yields the optimal parameters
				\[
				(c^{*},m^{*},k^{*},S_{c^{*},m^{*},k^{*}})
				\]
				for phase identification, together with the corresponding entropy gap $\Delta(c^{*},m^{*},k^{*},S_{c^{*},m^{*},k^{*}})$.
				
				\begin{remark}[\textbf{Optional Phase Constraints}]
					\label{rmk:opt_const}
					To ensure statistical significance and avoid pathological phase assignments where a single phase dominates entirely or the sequence is mostly undefined, we impose optional density constraints on the phase sets. Let $N$ be the total number of observed diagrams $D(\lambda)$. We require the proportion of diagrams in each phase to satisfy:
					$$
					\frac{|I^-|}{N} > \tau_{\text{disp}}, \quad \frac{|I^+|}{N} > \tau_{\text{cond}}
					$$
					where $|I^-|$ and $|I^+|$ denote the cardinalities of the dispersed and condensed phase sets, respectively, and $\tau_{\text{disp}}, \tau_{\text{cond}} \in [0,1]$ are user-defined tolerance hyperparameters.
				\end{remark}

				\subsection{Neural network training dynamics}\label{sec:nn}
				
				\subsubsection{Model and experimental setup}
				
				Artificial Neural Networks (ANNs) provide a computational framework inspired by the human brain, designed to learn complex patterns and relationships from data through the backpropagation algorithm and gradient descent optimization. The networks considered in this experiment are Convolutional Neural Networks (CNNs), which are specifically designed to process data with spatial structure, such as images or videos. CNNs apply learnable convolutional filters that slide across the input to extract increasingly complex features, ranging from edges and simple shapes to higher-level patterns. The resulting feature maps are then flattened and passed through fully connected layers to produce the final prediction.
				
				We use persistent homology to investigate how these convolutional filters self-organize geometrically throughout the optimization process. Specifically, following the framework introduced in \cite{gabrielsson2019exposition}, we analyze the evolution of persistent entropy over training iterations to track the topological organization of the learned filters. 

                A convolutional kernel generally consists of multiple spatial filters, one for each input channel. Following \cite{gabrielsson2019exposition}, we treat each spatial filter as an individual point in the point cloud. Therefore, a spatial filter of size $m\times n$ is flattened into a vector in $\mathbb{R}^{mn}$. Before topological analysis, each spatial filter is mean-centered and normalized, removing differences in offset and magnitude, represented as a point on the unit sphere.
                
                Collecting all flattened spatial filters at a given training iteration $I$ yields a point cloud, denoted by $X(I)\subset\mathbb{R}^{mn}$. Persistent homology is then computed on this point cloud using a Vietoris–Rips filtration induced by the Euclidean distance between spatial filters.
				
				Since the number of learned spatial filters can be very large, especially when aggregating spatial filters from multiple independently trained networks, we apply a density-based filtration following the methodology of \cite{gabrielsson2019exposition}. This procedure removes isolated low-density regions while preserving the dominant geometric structures of the spatial filters space. Let $X(I) \subset \mathbb{R}^{mn}$, be the point cloud and let $d_k(x)$ be the Euclidean distance from a point $x\in X(I)$ to its $k$-th nearest neighbor. We then define the density-filtered subset $X_{k, \tau}(I)\subset X(I)$ by retaining only the proportion $\tau$ of points with the smallest $d_k(x)$ values.
				
				\subsubsection*{MNIST dataset}
				
				For the first experiment, we consider the MNIST dataset, which contains $70,000$ grayscale images of handwritten digits from $0$ to $9$, each with a resolution of $28\times28$ pixels. The dataset was divided into $60,000$ training examples and $10,000$ test examples.
				
				The neural network architecture used in this experiment is the convolutional network $M(X,Y,Z)$ introduced in \cite{gabrielsson2019exposition}, whose configuration is summarized in Table~\ref{tab:cnn_m_architecture}. The architecture consists of two convolutional layers followed by a fully connected layer and a final Softmax readout layer. We use the configuration $M(64,32,64)$, containing $64$ kernels in the first convolutional layer and $32$ kernels in the second convolutional layer.
				
				\begin{table}[ht]
					\centering
					\caption{CNN architecture $M(X,Y,Z)$}
					\label{tab:cnn_m_architecture}
					\begin{tabular}{ll}
						\hline
						\textbf{Layer} & \textbf{Configuration} \\
						\hline
						Conv 1 & $3\times3\times X$ filters + ReLU + $2\times2$ Max Pooling \\
						Conv 2 & $3\times3\times Y$ filters + ReLU + $2\times2$ Max Pooling \\
						Fully Connected & $Z$ neurons + ReLU \\
						Readout & 10 neurons + Softmax \\
						Training & Cross-Entropy, Dropout ($p=0.5$), Adam ($lr = 10^{-3}$, weight decay = $0$) \\
						\hline
					\end{tabular}
				\end{table}
				
				Each spatial filter has size $3\times 3$ and is therefore represented as a nine-dimensional vector. We project each $3 \times 3$ spatial filter onto the unit sphere $S^8 \subset \mathbb{R}^9$ by first mean-centering and then normalizing its entries.
				
				We train $100$ independent realizations of the architecture $M(X=64,Y=32,Z=64)$ for $10,000$ training iterations, using a batch size of $128$. All trained networks reach an evaluation accuracy of approximately $99\%$.
				
				To study the evolution of the spatial filter geometry throughout training, we construct a point cloud every $100$ iterations, considering
$I\in\{100,200,\dots,10000\}$ for each convolutional layer. These point clouds are defined as follows:
				\begin{itemize}
					\item First convolutional layer, $X^{1}(I)$: Each of the $100$ trained networks contains $64$ spatial filters in the first layer. Pooling these spatial filters yields a point cloud $X^1(I) \subset \mathbb{R}^9$ containing $6,400$ points ($64 \times 100$).
					\item Second convolutional layer, $X^{2}(I)$: Each convolutional kernel receives $64$ input channels. Therefore, each kernel is decomposed into $64$ spatial filters of size $3\times 3$. Since the layer contains $32$ output kernels, each network contributes $32\times 64=2,048$ spatial filters. Aggregating the spatial filters from the $100$ trained networks gives $204,800$ points in $\mathbb{R}^{9}$, defining $X_2(I)$. 
				\end{itemize}
				
				Given a point cloud, we apply the density-filtration described previously. Due to the structural differences and sizes of the datasets, the filtration parameters are tailored for each layer:
				\begin{itemize}
					\item For $X^1(I)$, we set $k = 200$ and $\tau = 0.3$, retaining the $1,920$ densest points for subsequent topological analysis.
					\item For $X^2(I)$, we set $k = 10$ and $\tau = 0.1$, yielding a subset of $20,480$ points. Since this point cloud remains computationally expensive for Vietoris--Rips persistence, we further apply Ripser's greedy permutation subsampling, obtaining a representative set of $2,000$ landmark points for the persistent homology computation.
				\end{itemize}
				
				\subsubsection*{CIFAR-10 dataset}
				
				For the second experiment, we extend the previous analysis to the CIFAR-10 dataset, which contains $60,000$ color images belonging to $10$ object categories. Each sample consists of a $32\times 32$ RGB image, resulting in three input channels corresponding to the red, green, and blue components. Compared with MNIST, CIFAR-10 introduces a more complex visual structure due to the presence of natural images, multiple color channels, and higher intra-class variability.
				
				The dataset was divided into $50,000$ training images and $10,000$ test images. Following CIFAR-10 experiment in \cite{gabrielsson2019exposition}, the training images were randomly cropped to a spatial resolution of $24 \times 24$ pixels and randomly flipped horizontally with probability $0.5$. After these transformations, the images were converted into tensors and normalized independently for each RGB channel by subtracting the channel-wise mean and dividing by the corresponding standard deviation computed from the CIFAR-10 training set. The test images were center-cropped to $24\times 24$ pixels and then normalized using the same channel-wise statistics. 	
				
				The convolutional neural networks considered in this experiment follow the architecture $C(X,Y,Z)$ introduced in \cite{gabrielsson2019exposition} and summarized in Table \ref{tab:cnn_c_architecture}. In particular, we use the configuration $C(X=64,Y=32,Z=64)$ which consists of two convolutional layers with local response normalization, followed by a fully connected layer and a Softmax output layer.
				
				\begin{table}[ht]
					\centering
					\caption{CNN architecture $C(X,Y,Z)$}
					\label{tab:cnn_c_architecture}
					\begin{tabular}{ll}
						\hline
						\textbf{Layer} & \textbf{Configuration} \\
						\hline
						Conv 1 & $3\times3\times X$ filters + ReLU + $3\times3$ Max Pooling (stride 2) + LRN \\
						Conv 2 & $3\times3\times Y$ filters + ReLU + $2\times2$ Max Pooling (stride 1) + LRN \\
						Fully Connected & $Z$ neurons + ReLU \\
						Readout & 10 neurons + Softmax \\
						Training & Cross-Entropy, SGD ($lr=0.01$, momentum= $0.9$, weight decay=$10^{-4}$) \\
						\hline
					\end{tabular}
				\end{table}
				
				We train $60$ independent networks for $20,000$ training iterations using a batch size of $124$.
				
				To study the evolution of the learned convolutional spatial filters, we construct a point cloud at every $100$ training iterations. At each recorded iteration $I$, spatial filters from each convolutional layer are extracted and used to construct the corresponding point clouds.
				\begin{itemize}
					\item First convolutional layer, $X^{1}(I)$: Each kernel contains three spatial filters corresponding to the RGB channels. Therefore, each trained network contributes $64 \times 3 = 192$ spatial filters. Aggregating the spatial filters from the $60$ independently trained networks results in the point cloud $X^{1}(I)\subset\mathbb{R}^9$ containing $11,520$ points ($60\times 64\times 3$).
					\item Second convolutional layer, $X^{2}(I)$: The second convolutional layer maps $64$ input channels $32$ output channels. This structure results in $32 \times 64 = 2,048$ spatial filters of size $3 \times 3$ per network. Aggregating these filters across all $60$ trained realizations yields a large-scale point cloud $X^2(I) \subset \mathbb{R}^9$ containing $122,880$ points ($60 \times 32 \times 64$).
				\end{itemize}
				
				To reduce topological noise while preserving the most representative geometric structures, we again apply density filtration to the point clouds.
				
				\begin{itemize}
					\item For $X^{1}(I)$, we use the parameters $k=200$ and $\tau=0.14$, retaining the $1,612$ densest points for the subsequent topological analysis.
					
					\item For $X^{2}(I)$, the considerably larger size of the point cloud requires a more aggressive density filtration. We therefore set $k=15$ and $\tau=0.1$, obtaining a filtered subset of $12,288$ points. Since computing persistent homology on such a large point cloud remains computationally expensive, we apply again Ripser's greedy permutation subsampling to obtain a representative set of $2,000$ landmark points on which the persistent homology computation is performed.
				\end{itemize}

				\subsubsection{Results}
				
				We analyze the structural evolution of the spatial filters point cloud throughout the training process by tracking its underlying topological features.
				
				\subsubsection*{MNIST dataset}
				
				As shown in Figure~\ref{fig:mnist_accuracy}, we compute the mean accuracy across different training iterations and average it over all network instances. Standard deviation remains very small throughout training, indicating that the accuracy evolves consistently across different runs. A rapid increase in accuracy is observed during the early stages of training, up to approximately iteration $1,000$, after which it gradually stabilizes and converges to a final accuracy of $99\%$. This evolution in accuracy is closely related to the topological changes shown in Figure~\ref{fig:mnist_barcodes}. 
				
				\textbf{Layer 1}
				
				The dynamics of Layer $1$ are particularly interesting, as the formation of distinct structures can be clearly observed throughout the training process. At initial phase of training, when the accuracy is still low, the $H_1$ persistence barcodes for Layer $1$, consist mainly of short-lived intervals with no dominant topological features. As the accuracy begins to increase, one of the $H_1$ bars gradually becomes more prominent. Around iteration $300$, a distinct topological structure emerges, characterized by a highly persistent $H_1$ bar that clearly stands out from the background noise. This dominant topological feature remains stable throughout the rest of the training process, up to iteration $5,000$, suggesting that the spatial filters organize around a low-dimensional geometric structure containing a prominent one-dimensional cycle. From iteration $6,000$ onwards, when the test accuracy has already reached $99\%$, the geometry of the filter space changes, creating a second persistent cycle that remains until the final iteration and therefore leading to a space dominated by two circular structures.
				
				\textbf{Layer 2}
				
				For the second convolutional layer, the topological evolution differs from that observed in the first layer. At the beginning of training, the persistence barcodes already reveal the presence of a dominant $H_1$ feature, indicating that the spatial filters are organized around a relevant circular structure. However, as the classification accuracy begins to increase, multiple distinct, highly persistent $H_1$ bars simultaneously emerge and stand out from the topological noise. These multiple topological features remain clearly visible through to iteration $10,000$, although in the final iterations the relevant bars begin to become shorter. This behavior indicates that the deeper convolutional layer self-organizes into a more complex topological space, capturing multiple independent circular features or a higher-dimensional manifold with multiple loops.
				
				\begin{figure}
					\centering
					\includegraphics[width=0.6\textwidth]{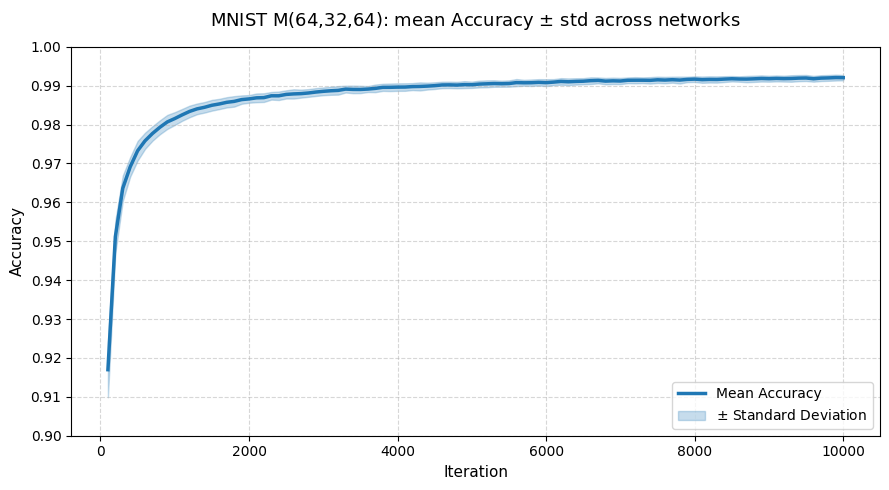}
					\caption{Mean accuracy ($\pm$ standard deviation) across $100$ neural networks at each training iteration.}
					\label{fig:mnist_accuracy}
				\end{figure}
				
				\begin{figure}
					\centering
					\includegraphics[width=1\textwidth]{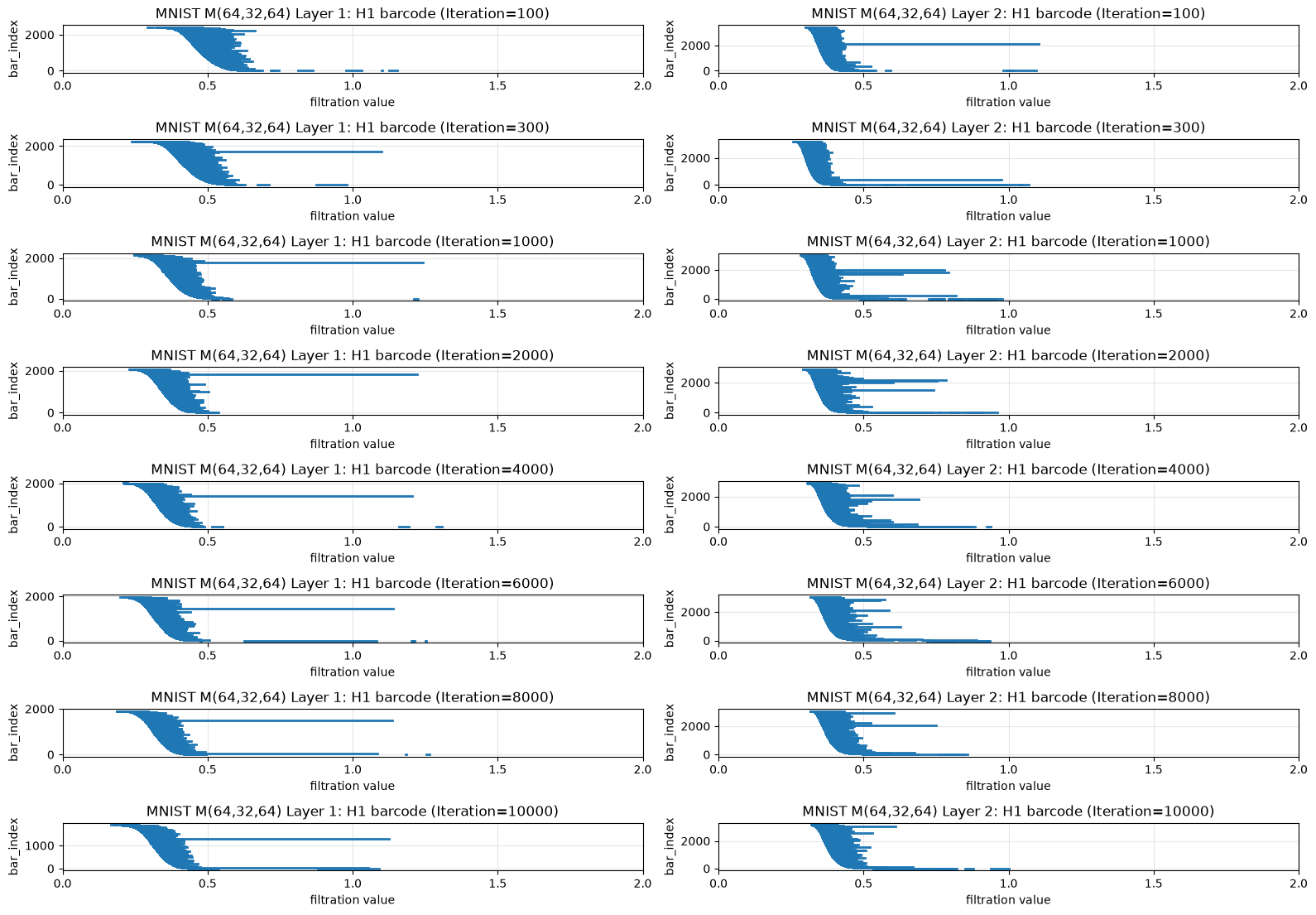}
					\caption{Persistence barcodes for dataset MNIST with architecture $M(64,32,64)$ at $8$ representative iterations for density filtration $p(k=200, \tau=0.3)$ for layer $1$ and density filtration $p(k=10, \tau=0.1)$ for layer $2$}
					\label{fig:mnist_barcodes}
				\end{figure}
				
				\subsubsection*{CIFAR-10 dataset}
				As shown in Figure~\ref{fig:cifar_accuracy}, the test accuracy increases gradually during approximately the first $10,000$ training iterations before reaching a more stable regime for the remainder of the optimization process. Compared with MNIST, the learning curve is noticeably slower, with the test accuracy remaining below $40\%$ during the early stages of training. This behavior reflects the greater complexity of the CIFAR-10 dataset and the more challenging classification task. The network eventually converges to a final test accuracy of approximately $77\%$. As in the MNIST experiments, this evolution in classification performance is accompanied by a progressive reorganization of the topology of the learned filters space, as illustrated by the persistence barcodes in Figure~\ref{fig:cifar_barcodes}.
				
				\textbf{Layer 1}
				
				The first convolutional layer exhibits a clear topological evolution throughout the training process. During the initial stages of training, the $H_1$ persistence barcodes are dominated by numerous short-lived intervals together with a few moderately persistent features, indicating a highly noisy and weakly organized filter space. As training progresses and the classification accuracy improves, the amount of topological noise gradually decreases, suggesting that the spatial filters become increasingly structured. Once the test accuracy reaches its plateau, around iteration $10,000$, the persistence barcode reveals a stable topological organization that remains essentially unchanged until the end of training. In particular, one highly persistent $H_1$ cycle clearly dominates the barcode, while a second cycle, born at a relatively late stage of the filtration, is consistently observed. Unlike the behavior observed for MNIST, these structures emerge more gradually and only become clearly distinguishable after the network has nearly converged, reflecting the greater complexity of the CIFAR-10 classification task.
				
				\textbf{Layer 2}
				
				The evolution of the second convolutional layer provides a less pronounced topological signature than that observed in the first layer. During the initial stages of optimization, the persistence barcodes are almost entirely composed of short-lived intervals, indicating that the space of spatial filters is largely dominated by topological noise with little evidence of meaningful global organization. As training progresses, several persistent $H_1$ features emerge, suggesting that the spatial filters begin to organize into multiple topological structures. Although moderately persistent cycles appear during the intermediate stages of optimization, their persistence progressively decreases in the later iterations. Towards the end of optimization, the dominant $H_1$ features decrease in persistence, indicating that the space of spatial filters does not condense into a small number of highly persistent structures. Instead, the learned representations retain a more distributed topological organization.

				\begin{figure}
					\centering
					\includegraphics[width=0.6\textwidth]{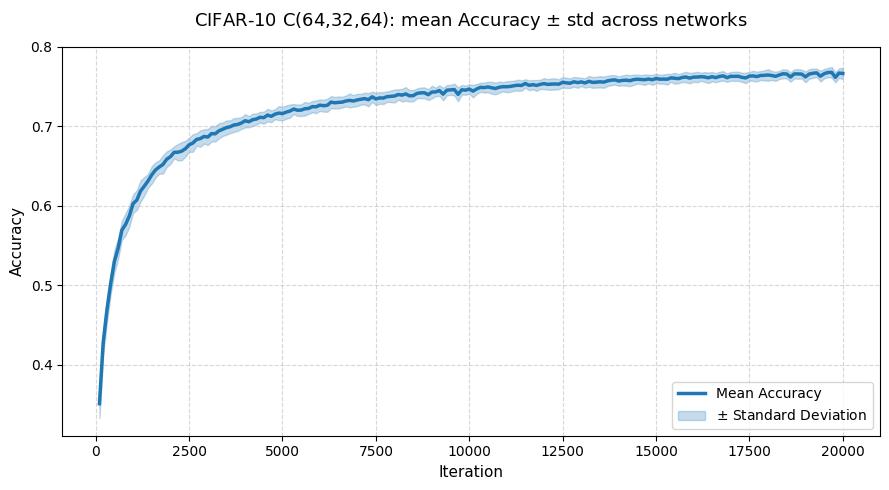}
					\caption{Mean accuracy ($\pm$ standard deviation) across $100$ neural networks at each training iteration.}
					\label{fig:cifar_accuracy}
				\end{figure}
				
				\begin{figure}
					\centering
					\includegraphics[width=1\textwidth]{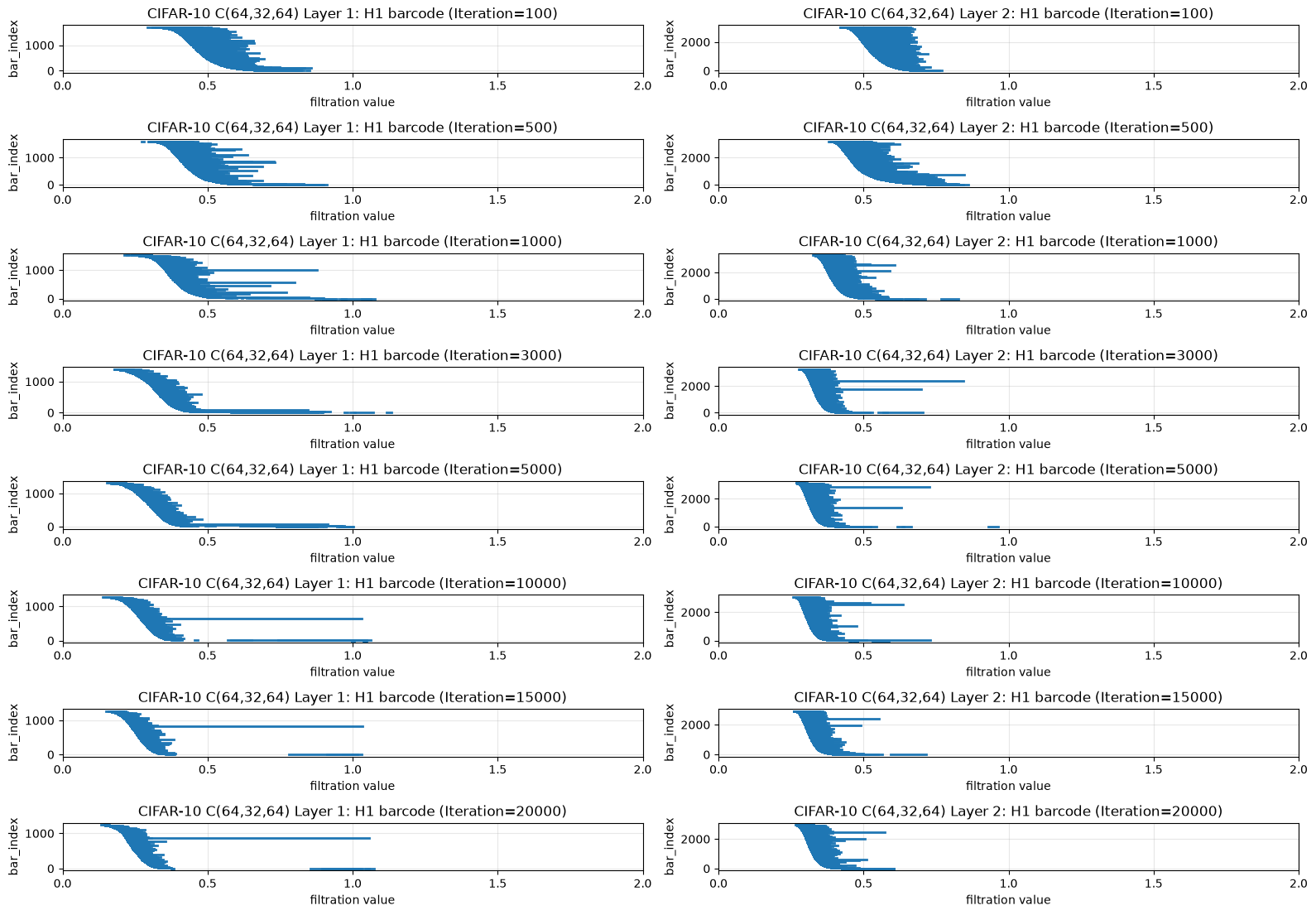}
					\caption{Persistence barcodes for dataset CIFAR-10 with architecture $C(64,32,64)$ at $8$ representative iterations for density filtration $p(k=200, \tau=0.14)$ for layer $1$ and density filtration $p(k=15, \tau=0.1)$ for layer $2$}
					\label{fig:cifar_barcodes}
				\end{figure}

				\subsubsection{Relation to the general theorem}
				The barcodes reported above reproduce, under the density filtration of~\cite{gabrielsson2019exposition}, the circular organization of spatial filters described in that work: at late training iterations a single $H_1$ interval dominates the first-layer barcode. What that observation does not settle is when the structure emerges, or on what basis one is entitled to call it emerged. Both are decided below by Theorem~\ref{thm:finite}. Rather than declaring a cycle present when a bar looks long, we ask whether the normalized persistence weights satisfy the condensation hypothesis, and we report the certified entropy gap that separates the resulting regime from the dispersed one. The training iteration at which the classification flips is then the onset of the transition, and it is this quantity, not the visual prominence of a bar, that we compare across datasets and layers.

				To relate the experimental observations to Theorem~\ref{thm:finite}, we identify the control parameter with the training iteration, $\lambda \equiv I$. For each iteration $I$, we construct a random persistence diagram $D(I)$ by applying Vietoris-Rips persistent homology to the density-filtered point cloud pooled from all the trained networks.
				
				\subsubsection*{MNIST dataset}
				
				\textbf{Layer 1}
				
				For first layer, as shown in the barcodes of Figure \ref{fig:mnist_barcodes}, the diagrams contain a significant amount of topological noise. To isolate the essential features, all persistence bars with lifetime smaller than $0.15$ are discarded before computing the persistent entropy. This value provides normalized persistence weights with a
				clear separation between relevant and negligible features, allowing the
				dispersion and condensation criteria of Definition \ref{def:disp_cond} to be applied consistently.
				
				The evolution of the persistent entropy is shown in Figure \ref{fig:mnist_evolution}. Consistent with the barcodes dynamics, the persistent entropy drops significantly during the early phase of optimization. This drop corresponds to the concentration of persistent mass into a small set of dominant bars, signaling a transition from a dispersed phase ($I^-$) to a condensed phase ($I^+$).
				
				To quantitatively characterize the dispersion-condensation framework, we apply the numerical detection procedure introduced in Section~\ref{sec:num_detection}. Inspection of the evolution of the normalized persistence weights shows that the diagrams never condense into more than two dominant cycles. Consequently, we impose the constraint $k_{\max}=2$. The optimization procedure identifies the following threshold parameters:
				\[
				c^* = 0.028, \quad m^* = 32, \quad k^* = 2, \quad S^* = 2
				\]
				These values maximize both the entropy gap and the coverage of the phase classification, yielding
				\[
				\Delta(c^*, m^*, k^*, S^*) = H_{\text{disp}}(c^*, m^*) - H_{\text{cond}}(c^*, k^*, S^*) = 2.6777 > 0.
				\]
				
				The resulting phase classification is shown in Figure~\ref{fig:mnist_phase_classification}. During the first training iterations, the persistence diagrams are classified as dispersed. However, due to the rapid consolidation of the space of spatial filters into a small number of dominant topological components, some normalized persistence weights become negligible while others increase substantially, leading the system into the condensed phase. According to our criterion, a diagram belongs to the condensed phase whenever the cumulative persistence of the $k^*=2$ most persistent features exceeds $1-c^*$. 
				As illustrated in Figure~\ref{fig:mnist_phase_classification}, the condensed regime consists of an initial interval in which a single dominant topological feature is present, resulting in zero persistent entropy, followed by a later interval in which the entropy stabilizes around $0.5$. This stabilization suggests the emergence of an additional dominant cycle. This behavior is consistent with the evolution of the persistence barcodes shown in Figure~\ref{fig:mnist_barcodes}, where a second prominent $H_1$ feature becomes visible during the later stages of training.
				
				Consequently, three distinct training regimes can be identified: an initial dispersed phase, $I^{-}=[0,100]$, a transition region, $I^{0}=(100,500]$, during which the persistent entropy decreases rapidly, and a condensed phase, $I^{+}=(500,10000]$. 
				
				Finally, the entropy gap inequality predicted by Theorem~\ref{thm:finite} is satisfied for every dispersed-condensed pair identified by the numerical procedure:
				\[
				\PE\bigl(D(\lambda_{I^-})\bigr)-\PE\bigl(D(\lambda_{I^+})\bigr)
				\ge
				\Delta(c^*,m^*,k^*,S^*).
				\]
				
				\bigskip
				
				\textbf{Layer 2}
				
				For the second convolutional layer, persistence intervals with lifetime smaller than $0.1$ are discarded before computing the persistent entropy in order to suppress topological noise. The entropy evolution shown in Figure~\ref{fig:mnist_evolution} closely follows the behavior of the persistence barcodes. During the earliest training iterations, the persistence mass is concentrated in only a few dominant intervals, resulting in relatively low entropy. As optimization progresses, several $H_1$ barcodes emerge simultaneously (Figure~\ref{fig:mnist_barcodes}), producing a more uniform distribution of the normalized persistence weights and, consequently, a significant increase in persistent entropy. This behavior reflects a transition towards a more dispersed topological regime.
				
				As in the first convolutional layer, we apply the numerical phase detection procedure described in Section~\ref{sec:num_detection}. We set $k_{\max}=4$, motivated by the distribution of the normalized persistence weights. Inspecting the persistence diagrams, we observe that the diagrams exhibiting the strongest concentration of persistence tend to have their total persistence mass dominated by at most four leading features. With this constraint, the optimization yields the parameters
				\[
				c^* = 0.032, \quad m^* = 17, \quad k^* = 4, \quad S^* = 4
				\]
				resulting in an entropy gap
				\[
				\Delta(c^*, m^*, k^*, S^*) = H_{\text{disp}}(c^*, m^*) - H_{\text{cond}}(c^*, k^*, S^*) = 0.7261 > 0
				\]
				
				In this layer, as shown in Figure~\ref{fig:mnist_barcodes}, during the initial stage of training, the persistence mass is concentrated in only a few dominant cycles. Consequently, the cumulative weight of the four largest persistence intervals satisfies the condensation criterion of Definition~\ref{def:disp_cond}, and these diagrams are classified as belonging to the condensed regime $I^{+}=[0,500)$.
				As training progresses, the persistence mass becomes progressively distributed among a larger number of comparable $H_{1}$ features. From approximately iteration $3000$ onward, many diagrams contain at least $m^*=17$ persistence intervals whose normalized weights exceed the optimal threshold $c^{*}=0.032$, satisfying the dispersion criterion and therefore being assigned to the dispersed regime $I^{-}=(3000,10000]$.
				
				The Figure \ref{fig:mnist_phase_classification} also shows some diagrams that remain unlabeled $I^{0}$ during the late stages of training. This does not indicate a failure of the method. Rather, these diagrams exhibit a highly distributed persistence profile but do not strictly satisfy the dispersion conditions. Consequently, the numerical procedure leaves them unclassified instead of forcing a phase assignment. For this reason, we interpret these unlabeled points as conservative outliers of the classification algorithm rather than indicators of a distinct dynamical behavior.
				
				Finally, every dispersed-condensed pair identified by the algorithm satisfies the entropy gap inequality predicted by Theorem~\ref{thm:finite},
				\[
				\PE\bigl(D(\lambda_{I^-})\bigr)-\PE\bigl(D(\lambda_{I^+})\bigr)
				\ge
				\Delta(c^*,m^*,k^*,S^*),
				\]
				providing additional numerical evidence supporting the finite-size entropy detection theorem.
				
				\begin{figure}
					\centering
					\includegraphics[width=0.8\textwidth]{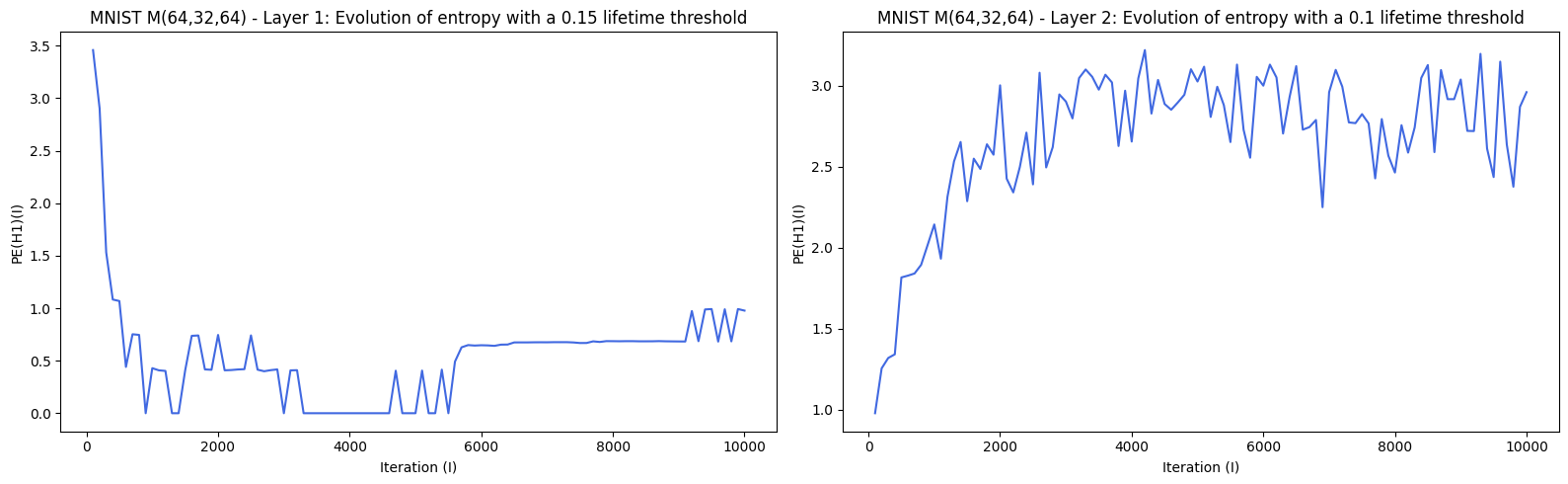}
					\caption{Entropy evolution for MNIST $M(64,32,64)$.}
					\label{fig:mnist_evolution}
				\end{figure}
				
				\begin{figure}
					\centering
					\includegraphics[width=0.8\textwidth]{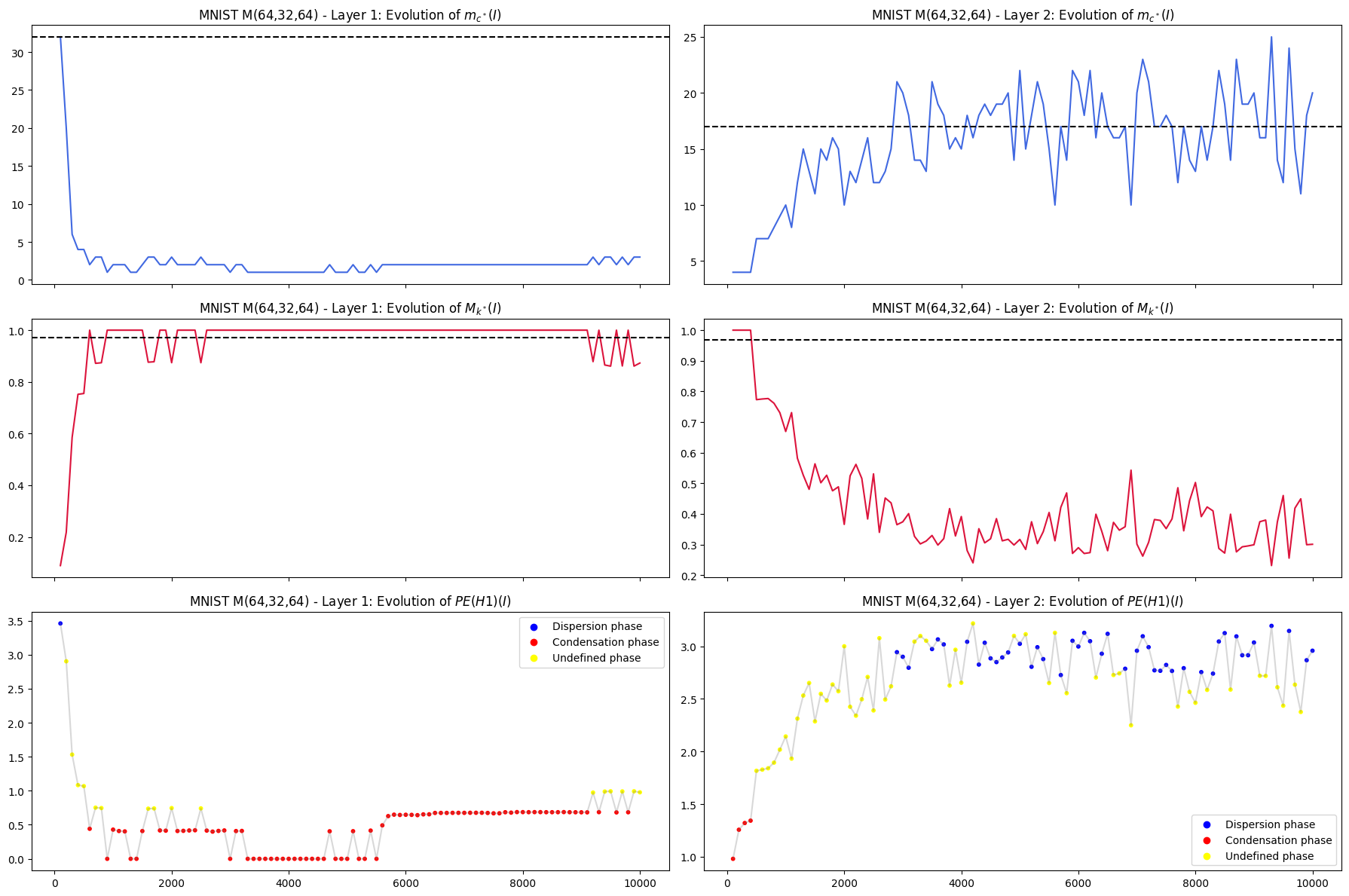}
					\caption{Evolution of $m_{c^*}(I)=\#\{i:p_{i}(I)\geq c^*\}$, $M_{k^*}(I) = \sum_{i=1}^{k^*}p_{i}(I)$ and phase classification between dispersed ($m_{c^*}(I)\ge m^*$) and condensed ($M_{k^*}(I)> 1-c^*$) regimes for both convolutional layers.}
					\label{fig:mnist_phase_classification}
				\end{figure}
				
				\subsubsection*{CIFAR-10 dataset}
				
				\textbf{Layer 1}
				
				For the first convolutional layer, persistence intervals with lifetime smaller than $0.13$ are discarded before computing the persistent entropy in order to suppress topological noise. As shown in Figure~\ref{fig:cifar_evolution}, the entropy decreases progressively throughout training, indicating that the normalized persistence weights become increasingly concentrated. This behavior is fully consistent with the evolution of the persistence barcodes shown in Figure~\ref{fig:cifar_barcodes}, where the initially noisy diagrams gradually develop a small number of dominant $H_1$ features.
				
				We next apply the numerical phase detection procedure introduced in Section~\ref{sec:num_detection}. Since the barcode evolution indicates that the condensed regime never contains more than three dominant persistence intervals, we set $k_{\max}=3$. Moreover, to avoid pathological solutions with an almost empty dispersed phase, we impose the optional coverage constraint $\tau_{\mathrm{disp}}=0.02$ (Remark~\ref{rmk:opt_const}). The optimization yields
				\[
				c^*=0.009,\qquad
				m^*=47,\qquad
				k^*=3,\qquad
				S^*=3,
				\]
				with entropy gap
				\[
				\Delta(c^*,m^*,k^*,S^*)
				=
				H_{\mathrm{disp}}(c^*,m^*)-
				H_{\mathrm{cond}}(c^*,k^*,S^*)
				=
				1.1647>0.
				\]
				
				The resulting classification, shown in Figure~\ref{fig:cifar_phase_classification}, identifies an initial dispersed regime, approximately $I^-=[0,500)$, in which at least $m^*=47$ persistence intervals carry normalized weights larger than the threshold $c^*=0.009$. As training progresses, the persistence mass gradually concentrates into only a few dominant topological features. Consequently, the algorithm detects a condensed regime that starts around iteration $11\,100$, approximately $I^+=[11100,20000)$, where the three largest persistence weights account for almost the entire persistence mass. This observation is consistent with the evolution of the barcodes shown in Figure~\ref{fig:cifar_barcodes}, where, from approximately iteration $10,000$ onward, two dominant persistence intervals emerge and remain stable until the end of training.
				
				Finally, every dispersed-condensed pair identified by the algorithm satisfies the entropy gap inequality predicted by Theorem~\ref{thm:finite},
				\[
				\PE\bigl(D(\lambda_{I^-})\bigr)-\PE\bigl(D(\lambda_{I^+})\bigr)
				\ge
				\Delta(c^*,m^*,k^*,S^*),
				\]
				providing numerical support for the finite-size entropy detection theorem.
				
				\bigskip
				
				\textbf{Layer 2}
				
				For the second convolutional layer, the same persistence threshold of $0.13$ is applied before computing the persistent entropy in order to suppress topological noise. As shown in Figure~\ref{fig:cifar_evolution}, the entropy is initially very high, reflecting persistence diagrams dominated by numerous low-persistence intervals. During training, several relevant topological features emerge and the entropy decreases accordingly. Towards the end of training, however, the entropy increases slightly again as the dominant persistence intervals become shorter and increasingly embedded within the surrounding topological noise, in agreement with the barcode evolution shown in Figure~\ref{fig:cifar_barcodes}.
				
			The diagrams exhibiting the strongest concentration of persistence mass show that the largest contributions are typically carried by the three most significant persistence weights. Therefore, we set $k_{\max}=3$ for the phase detection procedure. To avoid degenerate solutions, we impose the optional constraint $\tau_{\mathrm{disp}}=0.01$. The optimization procedure yields
				\[
				c^*=0.002,\qquad
				m^*=245,\qquad
				k^*=3,\qquad
				S^*=3,
				\]
				producing the positive entropy gap
				\[
				\Delta(c^*,m^*,k^*,S^*)
				=
				H_{\mathrm{disp}}(c^*,m^*)-
				H_{\mathrm{cond}}(c^*,k^*,S^*)
				=
				2.2769>0.
				\]
				
				The optimal parameters define a particularly stringent dispersion criterion, requiring at least $m^*=245$ persistence intervals to exceed the threshold $c^*=0.002$. Under these conditions, the algorithm identifies an initial dispersed regime, approximately $I^-=[0,400)$, followed by a condensed regime, approximately $I^+=[1200,5700)$. This behavior agrees with the persistence barcodes, which exhibit their clearest dominant cycles during the intermediate stages of training.
				
				During the final part of the optimization, most persistence diagrams remain unclassified and are assigned to the transition set $I^0$. This does not indicate the emergence of a new phase. Instead, the persistence diagrams no longer satisfy either the dispersion or the condensation conditions of Definition~\ref{def:disp_cond}. The dominant persistence intervals progressively lose prominence while the remaining persistence mass is not sufficiently distributed to satisfy the stringent dispersion criterion. Consequently, these late-training diagrams naturally fall into the undefined region $I^0$, reflecting a gradual loss of the clear dispersion--condensation structure observed earlier in training.
				
				As in the previous experiments, all classified dispersed-condensed pairs satisfy the entropy gap inequality of Theorem~\ref{thm:finite}, providing further numerical evidence for the validity of the theoretical framework.

				\begin{figure}
					\centering
					\includegraphics[width=0.8\textwidth]{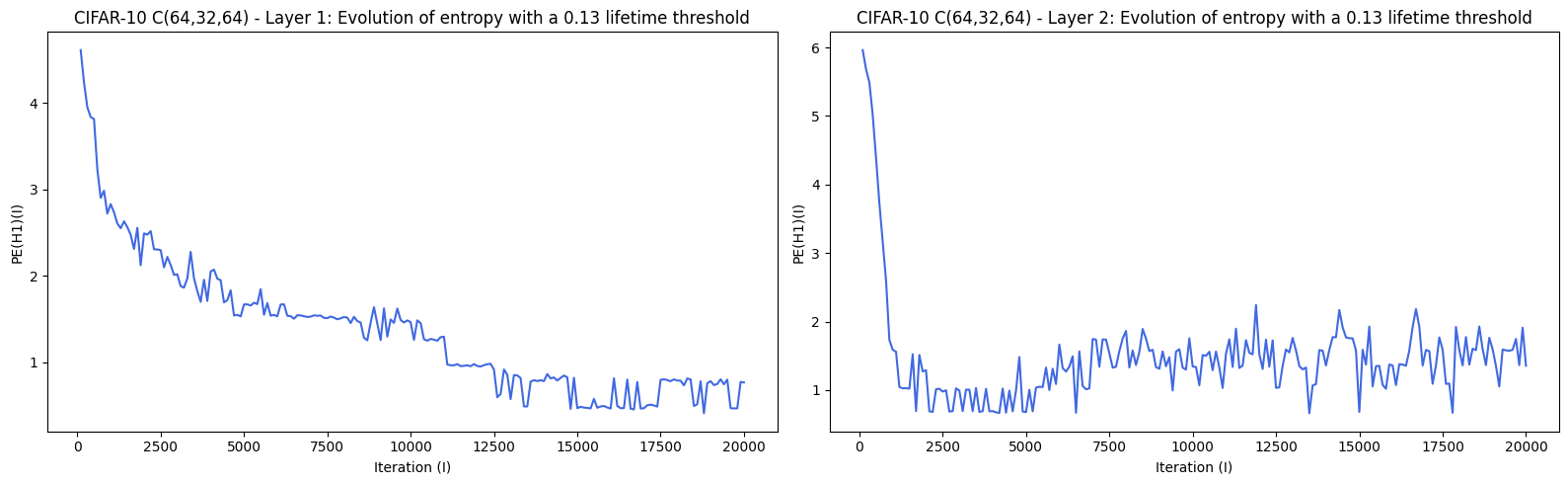}
					\caption{Entropy evolution for CIFAR-10 $C(64,32,64)$.}
					\label{fig:cifar_evolution}
				\end{figure}
				
				\begin{figure}
					\centering
					\includegraphics[width=0.8\textwidth]{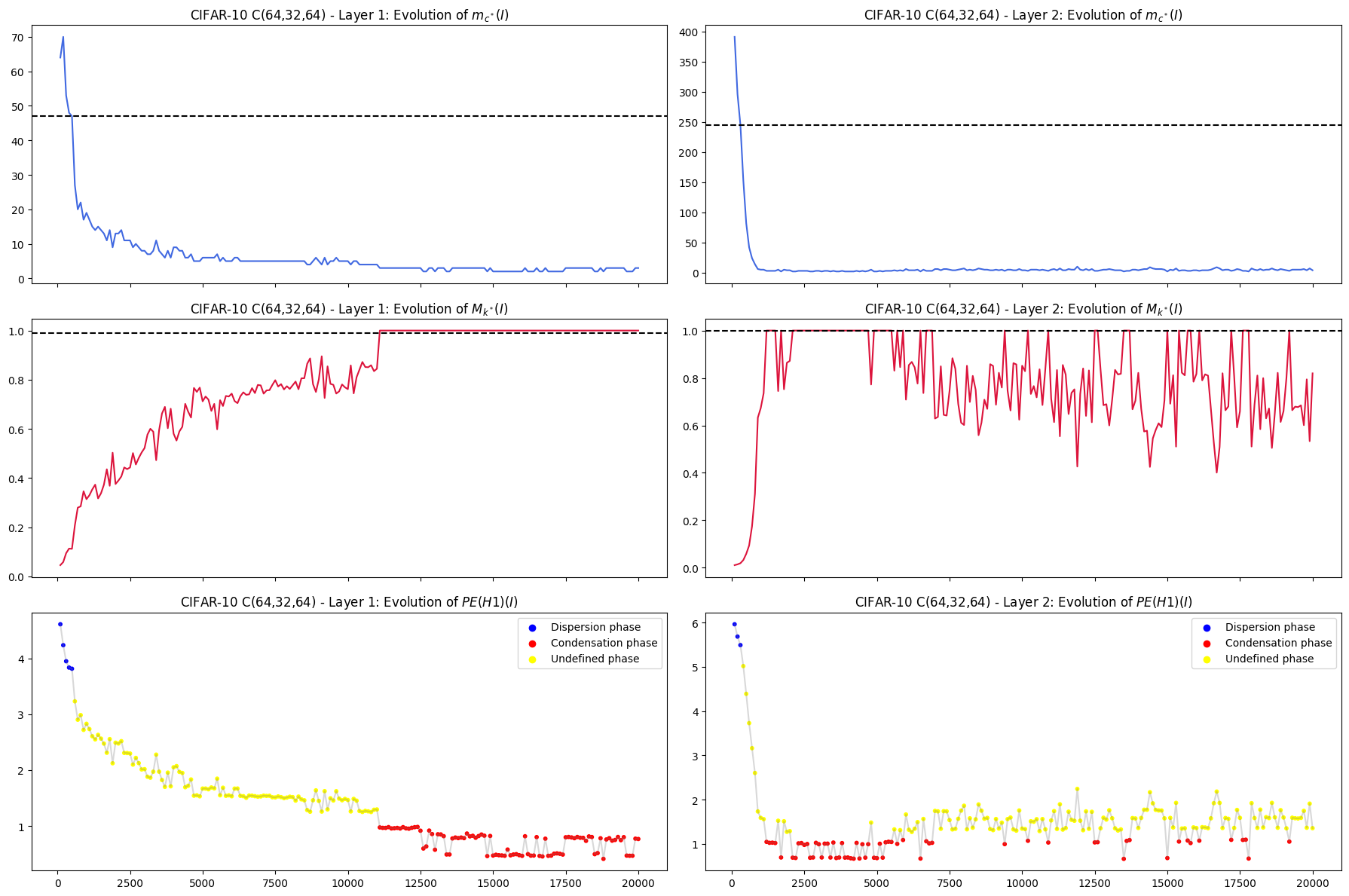}
					\caption{Evolution of $m_{c^*}(I)=\#\{i:p_{i}(I)\geq c^*\}$, $M_{k^*}(I) = \sum_{i=1}^{k^*}p_{i}(I)$ and phase classification between dispersed ($m_{c^*}(I)\ge m^*$) and condensed ($M_{k^*}(I)> 1-c^*$) regimes for both convolutional layers.}
					\label{fig:cifar_phase_classification}
				\end{figure}
				
				\subsection{Kuramoto synchronization}\label{sec:kuramoto}
				
				\subsubsection{Model and experimental setup}
				
				The Kuramoto model describes a network of coupled limit-cycle oscillators, where
				each oscillator's frequency is influenced by its intrinsic dynamics and
				interactions~\citep{english2007synchronization}. The dynamics is governed by
				\[
				\frac{d\theta_i}{dt} = 
				\omega_i + \frac{K}{N} \sum_{j=1}^{N} G_{i,j} \sin(\theta_j - \theta_i),
				\]
				where $\theta_i$ and $\omega_i$ are the phase and natural frequency of oscillator
				$i$, $K$ is the coupling constant, and $G_{i,j}$ is an adjacency matrix
				establishing the connectivity among oscillators. The degree of synchronicity is
				measured by the Kuramoto order parameter
				\[
				r = 
				\frac{1}{N}\sqrt{\left(\sum_{j=1}^{N} \cos(\theta_j)\right)^2 + \left(\sum_{j=1}^{N} \sin(\theta_j)\right)^2},
				\]
				which ranges from zero (incoherence) to one (full synchronization).
				
				Following the approach of~\citep{stolz2017persistent}, we construct functional
				networks from the pairwise synchronicity coefficient
				\[
				\phi_{i,j}^{T_k} = \langle | \cos(\theta_i^k - \theta_j^k) | \rangle,
				\]
				where $k$ indicates the time regime and the angular brackets denote the average
				across simulations. We simulate the Kuramoto model with $50$ oscillators,
				using random binary symmetric adjacency matrices without self-loops, initial
				frequencies sampled from a Gaussian distribution with mean~$0$ and standard
				deviation~$1$, and initial phases sampled uniformly in $(0,2\pi)$. Networks have
				$50$ nodes, $957$ edges, average degree~$38.28$, and density~$0.762$. The system
				is integrated over $500$ time points in $T=[0,10]$ with $\Delta t=0.02$ for
				coupling strengths $K\in[0,16]$ with step $0.5$, with $10$ repetitions per
				integration. The order parameter $r$ and pairwise synchronization $\phi_{i,j}$
				are computed for each $K$ and time point.

                To study the synchronization phases of the oscillators, we identify the control parameter with time, $\lambda\equiv t$. For each value of $K$, we construct a time-dependent persistence barcodes $D(K,t)$ by applying Vietoris--Rips persistent homology to the pairwise synchronicity matrix $\Phi(t)=(\phi_{i,j}(t))$, viewed as a distance matrix. For the essential infinite bar del diagram, the death value is taken as the maximum
distance in the corresponding distance matrix.

				\subsubsection{Results}
				
				Figure~\ref{fig:kuramoto_r} shows the temporal evolution of the Kuramoto order
				parameter. For low coupling strengths, the system remains in an incoherent
				regime with low and strongly fluctuating values of $r(t)$. As the coupling
				increases, the order parameter rapidly converges toward $r(t)\simeq 1$ with
				negligible variance, signaling the emergence of full synchronization.
				
				The corresponding behavior of the persistent entropy
				$\mathrm{PE}(H_0)(t)$ for different coupling strengths $K$ is shown in Figure~\ref{fig:kuramoto_pe}. For weak
				coupling strengths ($K=0$ and $K=1$), the entropy remains approximately constant at relatively large values. This behavior reflects the absence of global synchronization and the continuous rearrangement of the topological
				structure of the functional network. 
				
				At intermediate coupling strengths ($K=3$ and $K=5$), the entropy begins to decrease after an initial transient and exhibits larger temporal fluctuations. This regime corresponds to partial synchronization, where coherent oscillator clusters emerge and compete with incoherent regions. The broader variance observed across realizations indicates enhanced sensitivity to initial
				conditions.
				
				For strong coupling ($K=10$ and $K=15$), persistent
				entropy progressively stabilizes at a low constant value, indicating a collapse of topological complexity. Once synchronization is reached, the entropy remains
				nearly constant, indicating a stable coherent dynamical state.
				
				\begin{figure}
					\centering
					\includegraphics[width=0.6\textwidth]{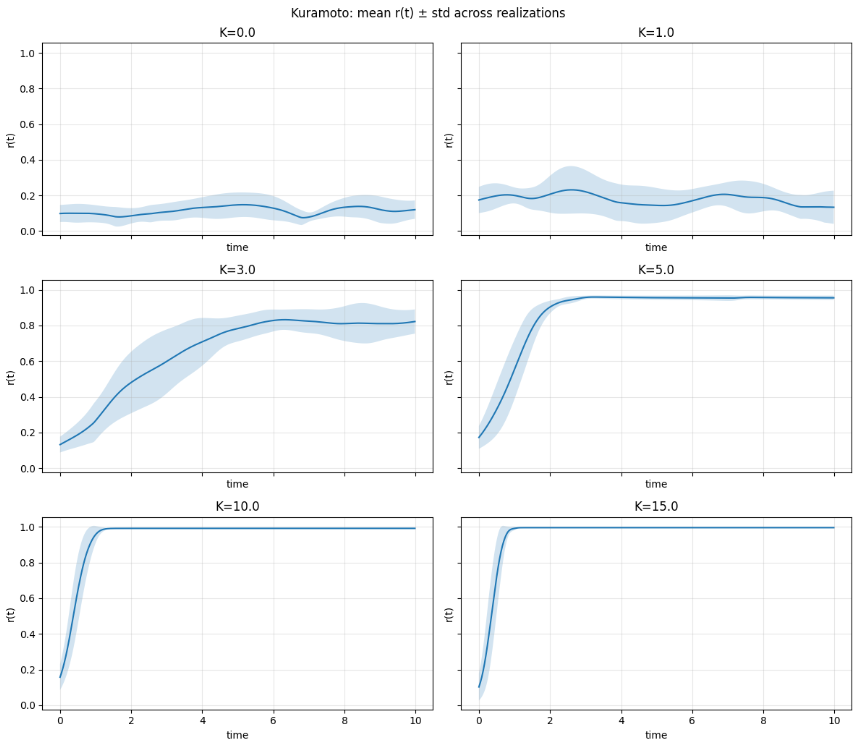}
					\caption{Kuramoto order parameter $r(t)$ for increasing coupling values $K$.
						Solid lines represent the mean across realizations, while shaded regions indicate the standard deviation. A sharp transition from incoherent dynamics to full synchronization is observed as $K$ increases.}
					\label{fig:kuramoto_r}
				\end{figure}
				
				\begin{figure}
					\centering
					\includegraphics[width=0.6\textwidth]{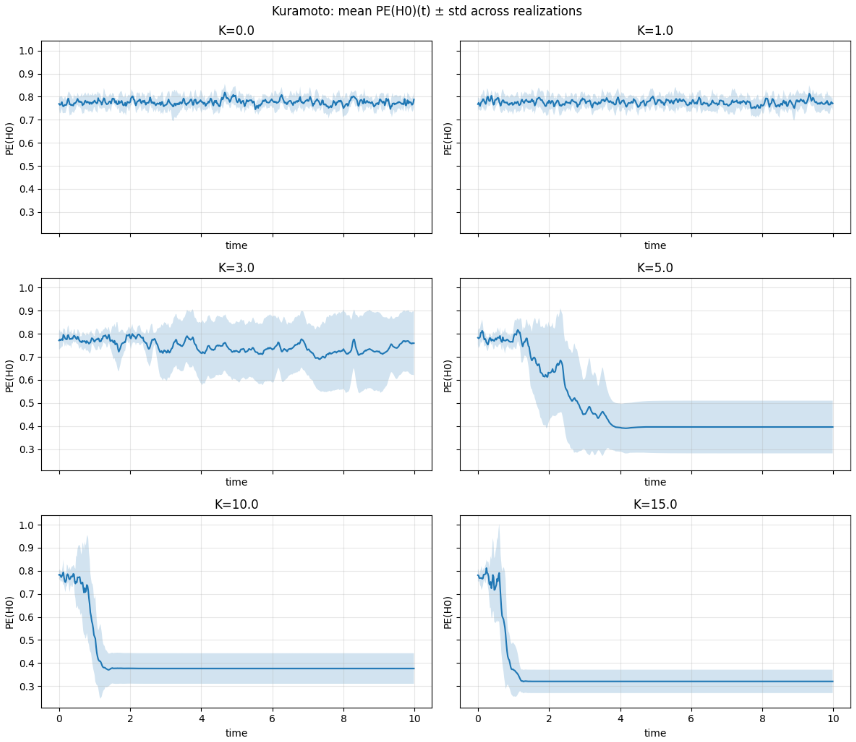}
					\caption{Temporal evolution of the persistent entropy $\mathrm{PE}(H_0)(t)$ for the Kuramoto model across different coupling strengths $K$. For low coupling, persistent entropy fluctuates around a high value, while for sufficiently large $K$ it rapidly drops and stabilises, indicating a collapse of topological complexity.}
					\label{fig:kuramoto_pe}
				\end{figure}
				
				\subsubsection{Relation to the general theorem}\label{sec:kuramoto-theorem}
				
				To illustrate the practical applicability of Definition~\ref{def:disp_cond} and Theorem~\ref{thm:finite}, we analyze a representative realization of the Kuramoto model with coupling strength $K=10$. This realization was selected because it exhibits a particularly clear synchronization process, allowing the dispersion--condensation mechanism to be identified unambiguously.
				
				Figure~\ref{fig:kuramoto_barcodes} shows persistence barcodes at several representative time points. Throughout the evolution, the persistence $H_0$ barcodes are characterized by one dominant bar that remains clearly separated from the remaining intervals. The other persistence intervals have much shorter lifetimes and correspond to topological noise. At early times, the dominant feature already contains a large fraction of the total persistence mass, although a non-negligible contribution is still distributed among several low-persistence intervals. As synchronization develops, these secondary lifetimes have already contracted, and an increasing proportion of the persistence becomes concentrated in the dominant bar, indicating that a stationary synchronized state has been reached.
				
				The $H_1$ persistence barcodes remain essentially empty throughout the evolution, with no significant one-dimensional persistent features. This suggests the absence of stable cyclic structures in the functional network, so the topological dynamics are driven almost entirely by changes in connectedness.
				
				\begin{figure}
					\centering
					\includegraphics[width=0.7\textwidth]{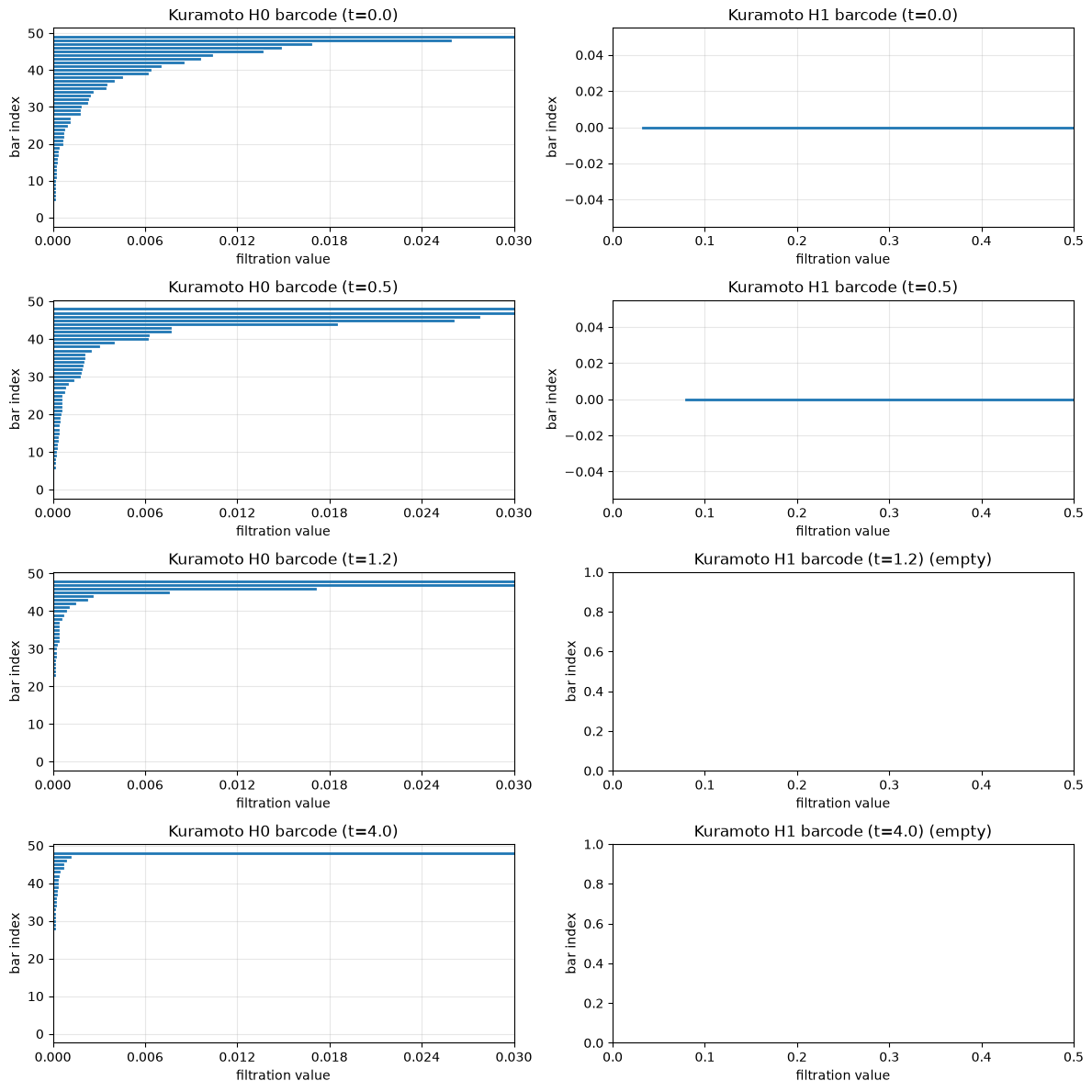}
					\caption{Persistence barcodes for the Kuramoto model at four representative time points ($t=0$, $t=0.5$, $t=1.2$ and $t=4$) for coupling strength $K=10$. Left panels show $H_0$ barcodes, while right panels show $H_1$ barcodes.}
					\label{fig:kuramoto_barcodes}
				\end{figure}
				
				For each persistence diagram $D(K,t)$, we evaluate the corresponding normalized persistence weights $p_i(K,t)$. Consistently with the evolution of the persistence barcodes shown in Figure \ref{fig:kuramoto_barcodes}, the persistent entropy remains low throughout the evolution and gradually decreases towards zero as the persistence mass becomes almost entirely concentrated in a single topological feature (Figure~\ref{fig:kuramoto_ev_entropy}). To reduce the influence of topological noise, lifetimes below a persistence threshold of $0.01$ are discarded.
				
				Applying the numerical detection procedure described in Section~\ref{sec:num_detection}, we set $k_{\max}=1$, since the barcodes in Fig.~\ref{fig:kuramoto_barcodes} show a single relevant persistent feature in the condensed phase. Additionally, we impose the constraint that the detected dispersed regime contains at least $10\%$ of the persistence diagrams. Under this constraint, the optimal parameters obtained are
				\[
				c^*=0.011, \qquad m^*=4, \qquad k^*=1, \qquad S^*=1.
				\]
				
				These parameters satisfy the hypotheses of Theorem~\ref{thm:finite} and yield a strictly positive entropy gap
				\[
				\Delta(c^*,m^*,k^*,S^*)
				=
				H_{\mathrm{disp}}(c^*,m^*)
				-
				H_{\mathrm{cond}}(c^*,k^*,S^*)
				=
				0.1813 >0. 
				\]
				
				The resulting classification successfully separates the dynamics into two well-defined regimes, as illustrated in Figure~\ref{fig:kuramoto_phases}: an initial dispersed phase $I^{-}=[0,1.18)$, characterized by the presence of at least four relevant persistence generators ($m^*=4$), and a later condensed phase $I^{+}=[1.4,10)$, dominated by a single persistent component ($k^*=1$).  Between these two regimes, a transition interval $I^{0}=[1.18,1.4)$ can be identified, which contains diagrams that are not classified as dispersed and exhibit unusually high persistent entropy values. This behavior arises because the normalized persistence weights associated with topological noise become extremely small, with not even four persistence weights exceeding the threshold of $0.011$.
				
				It can be verified that the theorem holds, since for all
				$t_{I^-} \in I^{-}$ and $t_{I^+} \in I^{+}$ it follows that
				\[
				\PE\bigl(D(K,t_{I^-})\bigr)
				- \PE\bigl(D(K,t_{I^+})\bigr)
				\;\ge\; \Delta(c^*,m^*,k^*,S^*)
				\]
				
				\begin{figure}
					\centering
					\includegraphics[width=0.4\textwidth]{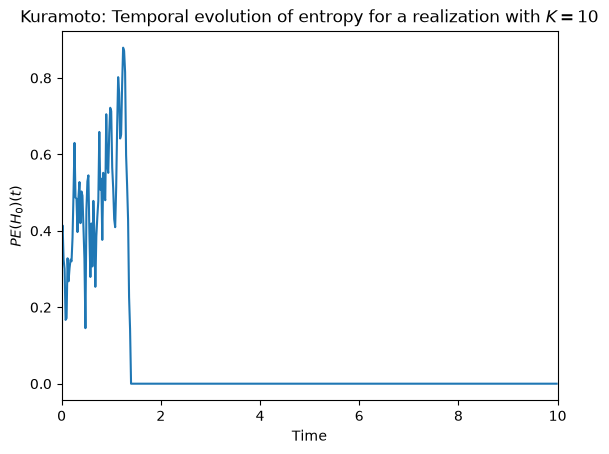}
					\caption{Temporal evolution of $\mathrm{PE}(H_0)(t)$ for a single realization with coupling strength $K=10$.}
					\label{fig:kuramoto_ev_entropy}
				\end{figure}
				
				\begin{figure}
					\centering
					\includegraphics[width=0.6\textwidth]{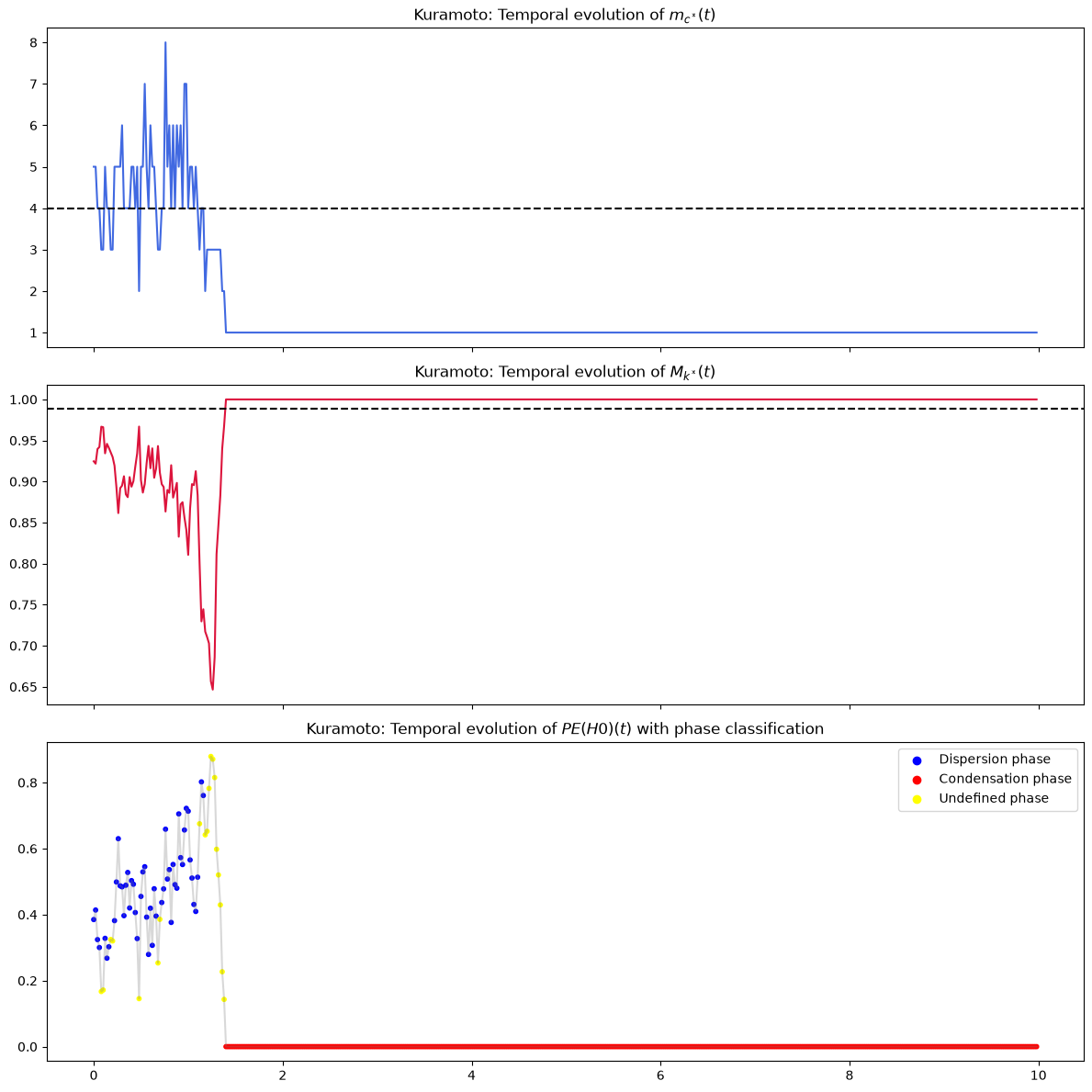}
					\caption{Temporal evolution of $m_{c^*}(t)=\#\{i:p_{i}(t)\geq c^*\}$, $M_{k^*}(t) = \sum_{i=1}^{k^*}p_{i}(t)$ and phase classification between dispersed ($m_{c^*}(t)\ge m^*$) and condensed ($M_{k^*}(t)> 1-c^*$) regimes.}
					\label{fig:kuramoto_phases}
				\end{figure}
				
				\subsection{Vicsek collective motion}\label{sec:vicsek}
				
				\subsubsection{Model and experimental setup}
				
				The Vicsek model is a paradigmatic minimal model for collective motion in active
				matter, describing $N$ self-propelled particles moving in a two-dimensional
				domain with periodic boundary
				conditions~\citep{Vicsek1995_Novel,Toner2005_Hydrodynamic}. Each particle $i$ is
				characterized by its position $\bm r_i(t)\in[0,L)^2$ and a velocity of constant
				modulus $v_0$ with direction $\theta_i(t)$. The dynamics is defined by the
				discrete-time update rules
				\[
				\theta_i(t+\Delta t) = \mathrm{Arg}\!\left( \sum_{j:\,\|\bm r_j(t)-\bm r_i(t)\|<r} e^{\mathrm{i}\theta_j(t)} \right) + \eta\,\xi_i(t), \quad
				\bm r_i(t+\Delta t) = \bm r_i(t) + v_0 \Delta t\,(\cos\theta_i(t),\sin\theta_i(t)),
				\]
				where $r$ is the interaction radius, $\eta$ is the noise amplitude, and
				$\xi_i(t)$ are independent random variables uniformly distributed in
				$[-\tfrac{1}{2},\tfrac{1}{2}]$. For low noise or high density, the system
				undergoes a transition from a disordered phase to an ordered flocking phase
				characterized by collective alignment of velocities. The standard measure of
				collective order is the polarization
				\[
				\psi(t) = \frac{1}{N} \left\| \sum_{i=1}^N (\cos\theta_i(t),\sin\theta_i(t)) \right\| \in [0,1].
				\]
				
				To probe the topological structure of Vicsek dynamics, we compute persistent
				homology from pairwise distance matrices derived from particle states. We
				consider two alternative distance definitions: a velocity-based distance
				\[
				d_{ij}^{(v)}(t) = \|\bm v_i(t) - \bm v_j(t)\|,
				\qquad 
				\bm v_i(t) = v_0(\cos\theta_i(t), \sin\theta_i(t)),
				\]
				and an orientation-based distance
				\[
				d_{ij}^{(\theta)}(t) = 1 - \left|\cos\big(\theta_i(t)-\theta_j(t)\big)\right|,
				\]
				both of which produce symmetric matrices suitable for Vietoris--Rips filtrations.
                
The Vicsek dynamics are simulated in a two-dimensional periodic square domain of side length
$L=5$, containing $N=300$ self-propelled particles with constant speed $v_0=1$.
The system is evolved using a time step of $\Delta t=0.1$ for $500$ iterations, corresponding
to a total simulation time of $T=50$. The interaction radius is fixed to $r=0.5$.
We investigate six noise amplitudes,
$\eta \in \{0.05,0.1,0.5,1,3,5\}$, and for each value of $\eta$ we generate five independent
realizations with random initial positions and orientations.

To characterize the topological evolution of the particle system, we compute persistent homology
from pairwise distance matrices obtained from the particle velocities. In accordance with
the theoretical framework, we identify the control parameter with the simulation time,
$\lambda \equiv t$. For each noise amplitude $\eta$, we construct
a time-dependent family of persistence diagrams $D(\eta,t)$, obtained by applying Vietoris-Rips persistent homology to the pairwise velocity-distance matrices. For the essential infinite $H_0$ interval, the death
value is replaced by the maximum entry of the corresponding distance matrix.

\subsubsection{Results}
				
                Dynamical evolution of $\psi(t)$ for increasing noise
				amplitudes is shown in Figure~\ref{fig:vicsek_psi}. For small values of $\eta$ ($\eta=0.05$ and $\eta=0.1$), the system rapidly develops global alignment
				and converges toward an ordered flocking state with $\psi(t)\simeq 1$ and
				minimal variability across realizations. As $\eta$ increases, convergence becomes
				slower and more heterogeneous, while for sufficiently large noise polarization remains strongly suppressed and fluctuating, indicating a disordered
				regime.
				
				\begin{figure}
					\centering
					\includegraphics[width=0.6\textwidth]{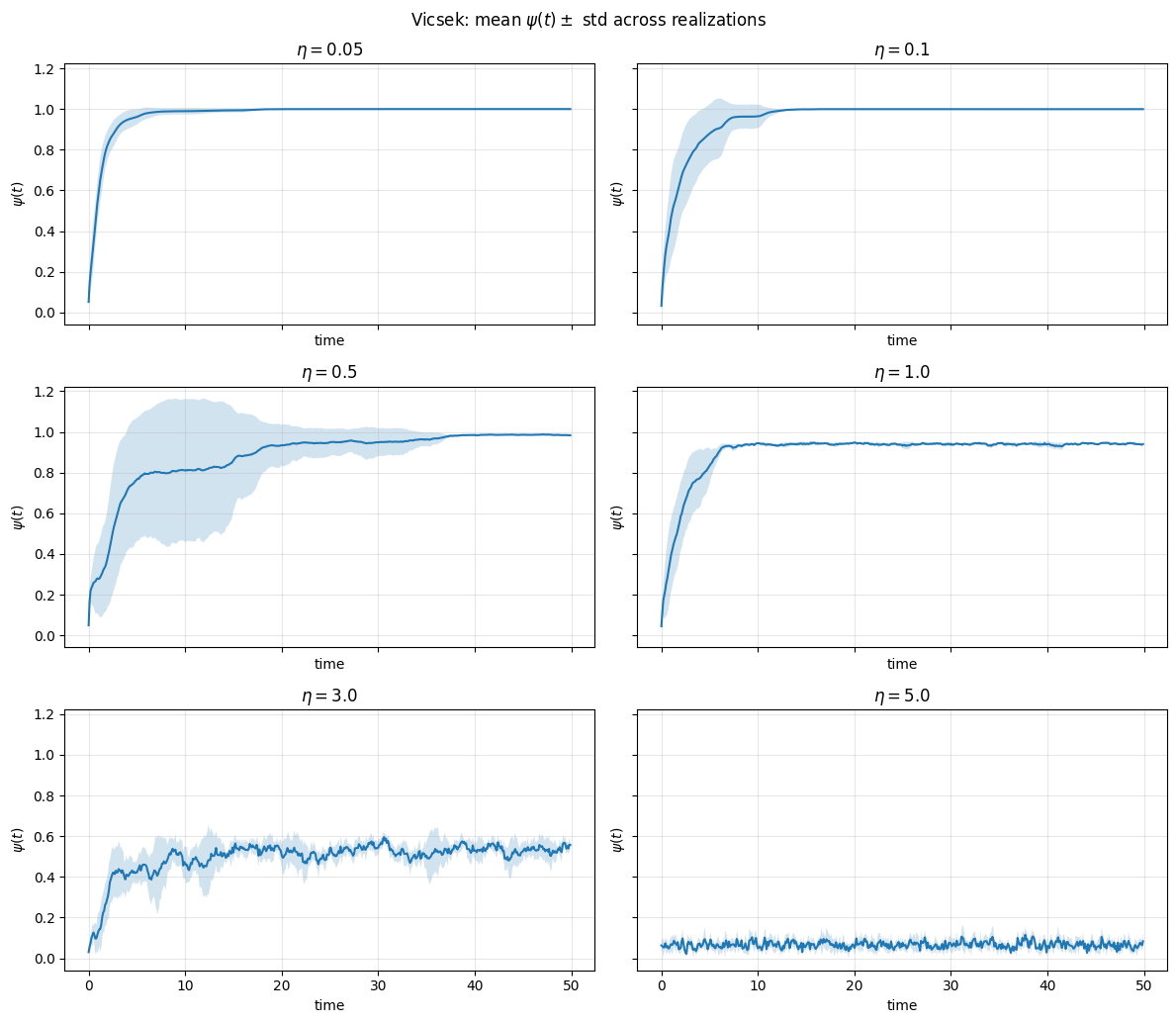}
					\caption{Mean Vicsek polarization $\psi(t)$ with standard deviation across
						realizations for different noise amplitudes $\eta$.
						Low noise leads to rapid convergence toward $\psi(t)\simeq 1$,
						while increasing noise delays or suppresses global ordering.}
					\label{fig:vicsek_psi}
				\end{figure}
				
				The corresponding topological evolution is captured by the temporal behavior of $\mathrm{PE}(H_0)(t)$ shown in Fig.~\ref{fig:vicsek_pe}. Because the persistent homology computed from the Vicsek dynamics contains a large number of short-lived topological features generated by topological noise, the persistent entropy was evaluated after applying a persistence threshold of $0.01$. This truncation removes low-persistence features that are associated with noise while preserving the dominant topological structures responsible for the collective dynamics. 
                
                For low noise amplitudes ($\eta=0.05$ and $\eta=0.1$), the entropy exhibits a rapid decay from initially large values toward a low-entropy stationary state. This behavior reflects the progressive reduction of topological complexity as particles become globally aligned and aggregate into a coherent flocking configuration.
				
				At intermediate noise levels ($\eta=0.5$ and $\eta=1.0$), the decay becomes
				slower and the entropy stabilizes at larger values with persistent temporal
				fluctuations. This indicates the coexistence of partially ordered and
				disordered structures, where local alignment competes with stochastic
				reorientation, preventing complete condensation into a single coherent state.
				
				In the high-noise regime ($\eta=3.0$ and $\eta=5.0$), the entropy remains approximately constant at high values throughout the simulation. The absence of
				significant decay reveals that strong noise continuously disrupts collective alignment, maintaining a highly fragmented and dynamically disordered configuration.

				\begin{figure}
					\centering
					\includegraphics[width=0.6\textwidth]{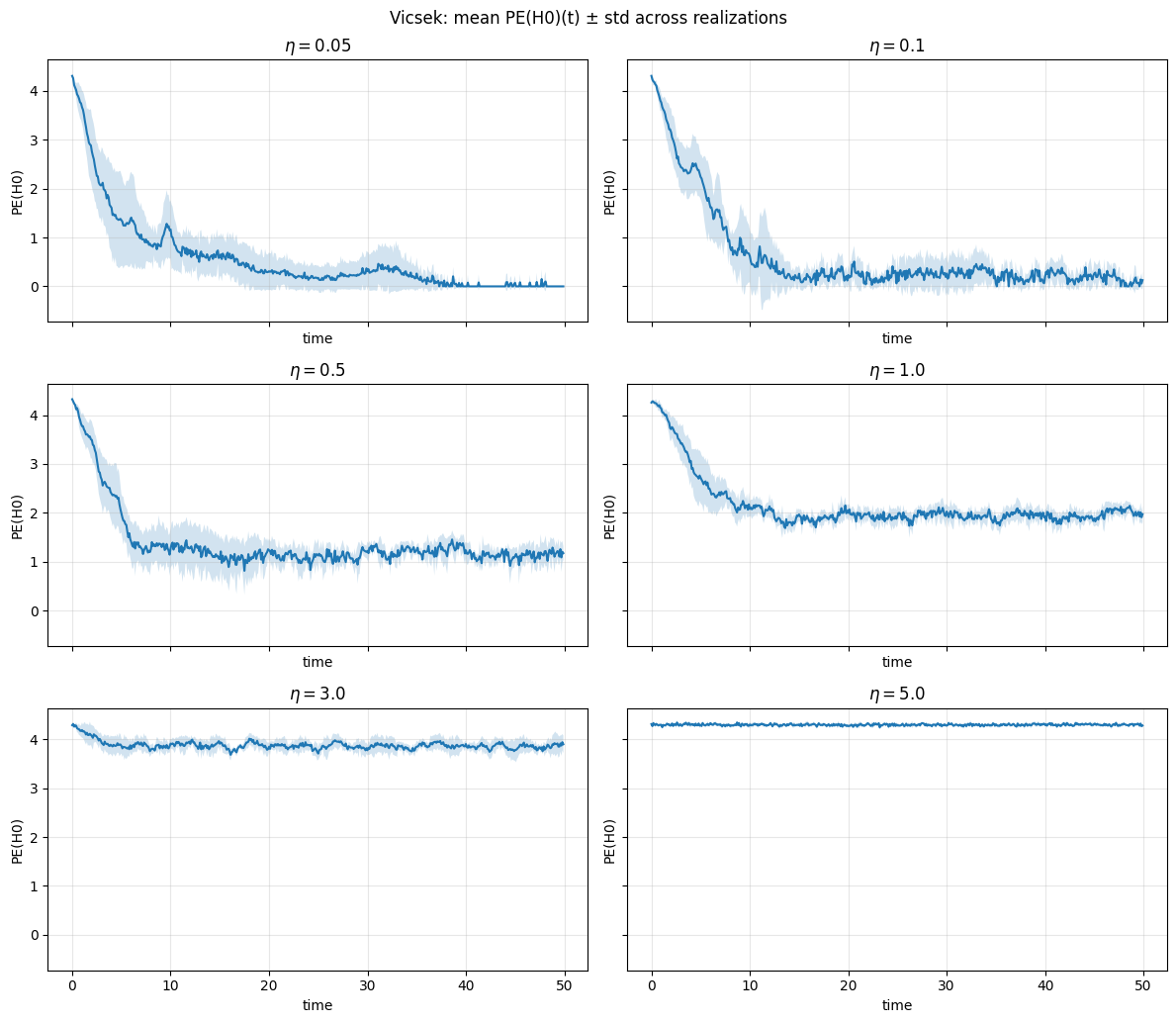}
					\caption{Temporal evolution of the persistent entropy
						$\mathrm{PE}(H_0)(t)$ for the Vicsek model across increasing noise amplitudes
						$\eta$. Low-noise regimes exhibit a rapid decay and stabilisation of persistent
						entropy, whereas higher noise levels are associated with sustained
						fluctuations.}
					\label{fig:vicsek_pe}
				\end{figure}

				\subsubsection{Relation to the general theorem}\label{sec:vicsek-theorem}
				
				Within the framework of Theorem~\ref{thm:finite}, we identify the control parameter $\lambda$ with the time variable $t$. To assess the applicability of the theoretical framework, we consider a representative realization at $\eta=0.05$. For each time step $t$, we extract the normalized persistence weights from the corresponding persistence diagrams.
				
				\begin{figure}
					\centering
					\includegraphics[width=0.7\textwidth]{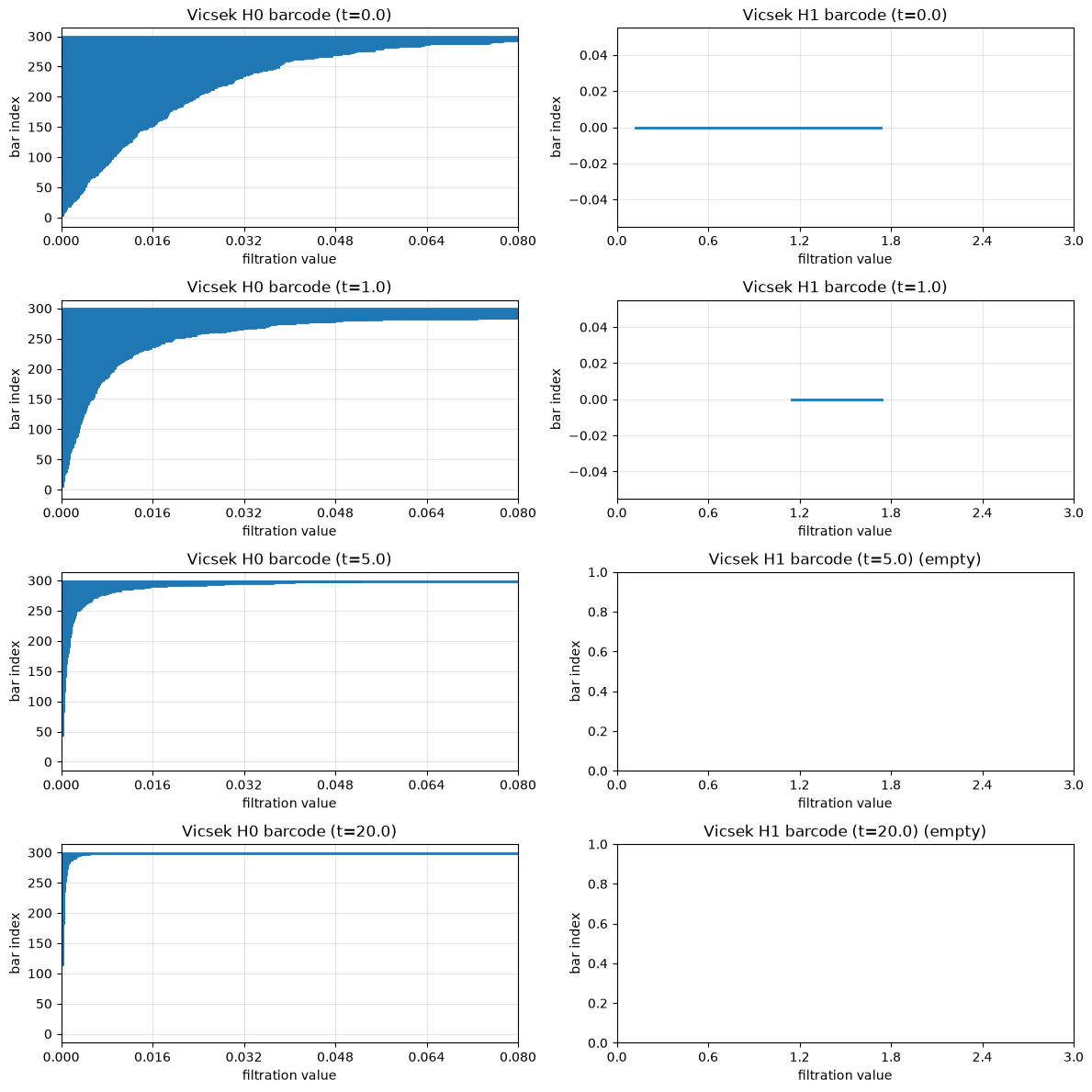}
					\caption{Persistence barcodes for the Vicsek model at four representative times ($t=0$, $t=1$, $t=5$, $t=20$) for noise amplitude $\eta=0.05$. Left panels show $H_0$ barcodes, while right panels show $H_1$ barcodes.}
					\label{fig:vicsek_barcodes}
				\end{figure}
				
				Figure~\ref{fig:vicsek_barcodes} illustrates the evolution of the persistence barcodes at low noise. Unlike the persistent entropy analysis, no persistence threshold is applied to the barcodes, allowing the complete topological evolution, including short-lived features associated with topological noise, to be visualized. The left column displays the $H_0$ barcodes, associated with connected components, whereas the right column shows the corresponding $H_1$ barcodes, associated with loop-like structures. Consistently with the entropy dynamics reported in Fig.~\ref{fig:vicsek_pe}, the flocking regime is accompanied by a progressive concentration of persistence into a single dominant connected component.
				
				At the initial time ($t=0$), the $H_0$ barcode exhibits a broad distribution of lifetimes, reflecting a fragmented configuration with many disconnected clusters. By $t=1$, most intervals have already contracted toward small filtration values, indicating the aggregation of particles into a coherent aligned state. At later times ($t=5$ and $t=20$), the $H_0$ barcode becomes nearly trivial, with connected components merging at very small filtration scales.

				The $H_1$ barcodes reveal a transient one-dimensional topological structure at the beginning of the dynamics. A single nontrivial $H_1$ feature is present at $t=0$ and $t=1$, while the $H_1$ barcodes are empty at all later sampled times. This behavior indicates that loop-like structures are short-lived and become unstable as global alignment emerges.
				
				\begin{figure}
					\centering
					\includegraphics[width=0.4\textwidth]{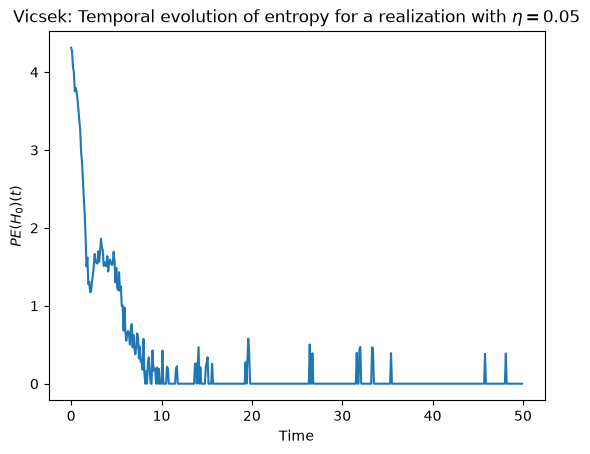}
					\caption{Temporal evolution of $\mathrm{PE}(H_0)(t)$ for a single realization with $\eta=0.05$}
					\label{fig:entropy_0.05}
				\end{figure}
				
				The persistent entropy displayed in Figure~\ref{fig:entropy_0.05}, computed after applying a persistence threshold of $0.01$ to filter out low-persistence features associated with topological noise, exhibits a rapid decay from an initially large value to a regime of very low entropy. This behavior reflects the emergence of a dominant persistence interval that progressively concentrates most of the persistence weight. After this transient, the entropy remains close to zero, with occasional fluctuations caused by finite-size effects and topological noise.
				
			Applying the numerical detection procedure described in Section~\ref{sec:num_detection} with $k_{\max}=1$, and, as in the Kuramoto experiment, imposing the additional constraint that the dispersed regime contains at least $10\%$ of the persistence
diagrams, we obtain the optimal phase-separation parameters
				\[
				c^* = 0.009,\qquad m^*=11,\qquad k^*=1,\qquad S^*=1.
				\]
                yielding the following entropy gap
				\[
				\Delta(c^*,m^*,k^*,S^*) = H_{\text{disp}}(c,m)-H_{\text{cond}}(c,k,S) = 0.5098.
				\]
				
				Figure~\ref{fig:phases_vicsek} shows the evolution of $m_{c^*}$ and $M_{k^*}$. Since $k^*=1$, the classification is completely determined by the most persistent interval. During the dispersion phase, at least $m^*=11$ intervals exceed the threshold $c^*$, leading to $m_{c^*}(t)\geq 11$ for all $t\in I^{-}$ In contrast, the condensation phase is characterized by a dominant interval carrying all the persistence mass, so that $M_{k^*}(t)>1-c$ and $m_{c^*}(t)=1$ for all $t\in I^{+}$.

				  From the Figure \ref{fig:phases_vicsek}, two clearly distinct phases can be identified:
				\[
				I^{-}=[0,5.6),
				\qquad
				I^{+}=[9,50].
				\]
                and the transition phase, $I^0$ between these two regimes.
				
				Finally, for every pair of time instants $t_{I^-}\in I^{-}$ and $t_{I^+}\in I^{+}$, the persistent entropy satisfies
				\[
				\PE\bigl(D(\eta,t_{I^-})\bigr)
				- \PE\bigl(D(\eta,t_{I^+})\bigr)\,\ge\,\Delta(c,m,k,S).
				\]
				
				Hence, the observed entropy difference between the dispersion and condensation regimes is uniformly bounded from below by the theoretical gap predicted by the theorem, providing direct numerical evidence for the phase-separation mechanism established in Section~\ref{sec:num_detection}.
				
				\begin{figure}
					\centering
					\includegraphics[width=0.6\textwidth]{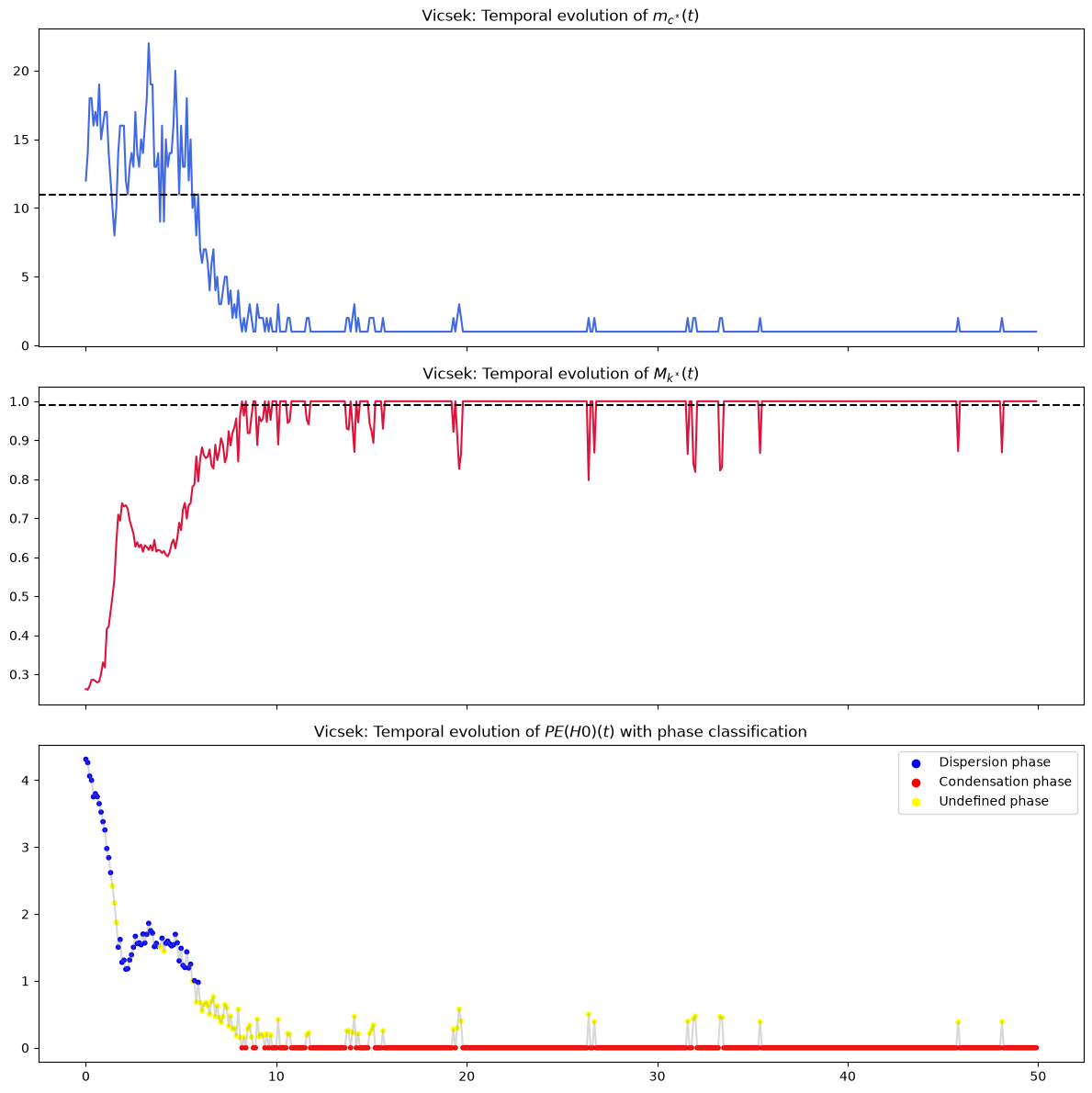}
					\caption{Temporal evolution of $m_{c^*}(t)=\#\{i:p_{i}(t)\geq c^*\}$, $M_{k^*}(t) = \sum_{i=1}^{k^*}p_{i}(t)$ and phase classification between dispersed ($m_{c^*}(t)\ge m^*$) and condensed ($M_{k^*}(t)> 1-c^*$) regimes.}
					\label{fig:phases_vicsek}
				\end{figure}

					\section{Conclusions}
					
					We have established a model-agnostic criterion showing that persistent entropy
					detects phase transitions whenever the transition reorganizes the lifetime distribution of the associated persistence diagrams from a
					\emph{dispersed} configuration into a \emph{condensed} one. If at least $m$ bars
					each carry a normalized weight of at least~$c$ in one regime, while in the other
					the $k<m$ most persistent bars concentrate more than $1-c$ of the total
					normalized persistence mass within a diagram of at most $S$ off-diagonal points, then
					Theorem~\ref{thm:finite} separates the two regimes by the explicit gap
					$\Delta(c,m,k,S)>0$ with probability at least $1-2\delta$ at any finite system
					size, and Corollary~\ref{cor:asymptotic} shows the separation is inherited by
					the limiting diagrams after truncation at a generic threshold. Since persistent
					entropy is invariant under a global rescaling of lifetimes, it can only register
					changes in the \emph{shape} of the weight distribution; the
					dispersion--condensation mechanism is precisely that change, which explains both
					why persistent entropy succeeds as a phase detector across unrelated systems and
					what it is detecting. The guaranteed sensitivity is computable in advance from
					$(c,m,k,S)$ alone, with no reference to the filtration, the embedding, the
					homological degree, or the system size.
					
					Because the hypotheses are stated entirely in terms of normalized weights, they
					are directly checkable on empirical barcodes. The numerical procedure of
					Section~\ref{sec:num_detection} exploits this: it labels each observed diagram as
					dispersed, condensed, or unclassified, and selects the parameters that trade the
					guaranteed gap against the fraction of diagrams the classification can label,
					returning a certified lower bound on the entropy separation for the system at
					hand. Diagrams satisfying neither hypothesis are left unlabeled rather than
					forced into a phase, so the method is conservative by construction.

					The principal application is to the topology of convolutional filters.
					Gabrielsson and Carlsson~\cite{gabrielsson2019exposition} established that
					trained filters organize around circular structures; our results show that
					this organization is reached through a phase transition, and identify when.
					The first layers of $M(64,32,64)$ on MNIST and $C(64,32,64)$ on CIFAR-10 both
					undergo a clean dispersion-to-condensation transition, with certified gaps of
					$2.68$ and $1.16$, but on very different timescales: MNIST condenses within the
					first few hundred iterations, CIFAR-10 only after roughly eleven thousand, well
					past the point at which test accuracy has plateaued. The onset of topological
					organization is therefore not simply a by-product of convergence in loss. The
					second MNIST layer traverses the same mechanism in reverse, dispersing into
					several coexisting $H_1$ cycles as training proceeds, which the theorem
					accommodates unchanged, since it constrains the two regimes rather than the
					direction in which the control parameter sweeps them. That the criterion is not
					specific to learning dynamics is confirmed by the Kuramoto and Vicsek models,
					which exhibit the corresponding $H_0$ condensation at the onset of
					synchronization and of flocking. In every case the measured entropy difference
					exceeded the theoretical gap.
					
					The conditions are sufficient, not necessary: a transition that reorganizes a
					barcode without redistributing normalized persistence mass is invisible to
					persistent entropy, so a missing entropy gap is diagnostic rather than merely
					negative, identifying which hypothesis fails. This behavior is visible in some of our phase-classification plots presented in the experimental results, where persistence diagrams that satisfy neither condition are correctly left unclassified. The two hypotheses are currently coupled through a
					single constant~$c$, which forces the condensed regime to absorb more than
					$1-1/m$ of the persistence mass; decoupling them should yield a sharper gap.
					Developing hypothesis tests for dispersion and condensation from finitely many
					realizations, propagating the estimation error in $(c,m,k,S)$ into the
					guaranteed gap, and characterizing which filtrations induce
					dispersion-condensation structure for a given data modality all appear
					tractable within the present framework.

                \section*{Code Availability}
                
                The source code used to generate the numerical results presented in this work is available  in the following GitHub repository:
                
                \url{https://github.com/MarcosGutPozo/Persistent-Entropy-as-a-Detector-of-Phase-Transitions}.

				\bibliographystyle{plain}
				
				\bibliography{references}

			\end{document}